\documentclass{article}

\usepackage[preprint]{neurips_2026}

\usepackage[utf8]{inputenc} 
\usepackage[T1]{fontenc}    
\usepackage{hyperref}       
\usepackage{url}            
\usepackage{booktabs}       
\usepackage{amsfonts}       
\usepackage{nicefrac}       
\usepackage{microtype}      
\usepackage{xcolor}         
\usepackage[table]{xcolor}

\usepackage{array,tabularx}
\usepackage{makecell}
\usepackage{booktabs,caption}
\usepackage{multicol} 
\usepackage{multirow}
\usepackage[flushleft]{threeparttable}
\usepackage[most]{tcolorbox}
\usepackage{verbatim}

\usepackage{graphicx}
\usepackage{subcaption}
\usepackage{wrapfig}

\usepackage{algorithm}
\usepackage{algorithmic}

\usepackage{dsfont}
\usepackage{amsmath}
\usepackage{amssymb}
\usepackage{mathtools}
\usepackage{amsthm}

\theoremstyle{plain}
\newtheorem{theorem}{Theorem}[section]
\newtheorem{proposition}[theorem]{Proposition}
\newtheorem{lemma}[theorem]{Lemma}
\newtheorem{corollary}[theorem]{Corollary}
\theoremstyle{definition}
\newtheorem{definition}[theorem]{Definition}

\theoremstyle{remark}

\title{Delayed Homomorphic Reinforcement Learning \\for Environments with Delayed Feedback}

\author{%
  Jongsoo Lee \\
  Department of Convergence IT Engineering\\
  POSTECH\\  
  Pohang, 37673, South Korea\\
  \texttt{jjongs@postech.ac.kr} \\
  \And
  Jangwon Kim \\
  Department of Electrical and Information Engineering\\
  Seoul National University of Science and Technology\\  
  Seoul, 01811, South Korea\\
  \texttt{jangwonkim@seoultech.ac.kr} \\
  \And
  Soohee Han\thanks{Corresponding Author} \\
  Department of Electrical Engineering\\
  POSTECH\\
  Pohang, 37673, South Korea\\
  \texttt{sooheehan@postech.ac.kr} \\
}

\begin{document}

\maketitle

\begin{abstract}
    Reinforcement learning in real-world systems often involves delayed feedback, which violates the Markov assumption and impedes both learning and control. Canonical augmentation-based approaches address this issue by augmenting the state with action histories, but suffer from state-space explosion and the resulting sample-complexity burden. Despite recent progress, state-of-the-art augmentation-based methods either primarily alleviate this burden for the critic or employ separate mechanisms for the actor and critic. To address this limitation, we propose delayed homomorphic reinforcement learning (DHRL), a framework grounded in MDP homomorphisms that introduces a belief-equivalence relation over the augmented state space to collapse control-redundant augmented states into a compact abstract state space. This framework yields exact abstraction under deterministic dynamics and approximate abstraction under stochastic dynamics, providing a unified abstraction mechanism that benefits both the actor and critic. In finite domains, exact abstraction preserves optimality and recovers the delay-free sample-complexity order, while approximate abstraction admits a bound on the value loss of the resulting policy. For continuous domains, we introduce deep delayed homomorphic policy gradient (D$^2$HPG), a deep actor-critic instantiation of the DHRL framework. Experiments on MuJoCo continuous-control benchmarks demonstrate that D$^2$HPG outperforms strong augmentation-based baselines, especially in long-delay regimes. Code is available at \url{https://github.com/JongsooLee/DHRL}.
\end{abstract}

\section{Introduction}

Despite the remarkable success of reinforcement learning (RL) across diverse domains \citep{dqn, sac, RLHF, tokamak, parkour2}, real-world sequential decision-making systems often involve unavoidable delays arising from sensing, actuation, and communication latencies \citep{NCS, robotic-delay, kaufmann}. Such delays violate the Markov assumption underlying standard Markov decision processes (MDPs) \citep{MDP}, inducing non-Markovian dynamics that impede learning and destabilize behavior at inference time \citep{hwangbo, mahmood, acting}.

A canonical approach to delayed feedback is state augmentation, which restores Markovian property by incorporating action histories into the state representation \citep{augmented2, augmented}. However, this theoretically grounded approach leads to state-space explosion as the number of delayed time-steps increases, resulting in a severe sample-complexity burden \citep{mbs, acting, DIDA}. Although several remedies have been proposed, state-of-the-art augmentation-based methods either primarily alleviate this burden for the critic or employ separate mechanisms for the actor and critic \citep{BPQL, signal, AD-SAC, conservative-RL}. These limitations underscore the need for a unified and sample-efficient framework that addresses state-space explosion for both the actor and critic.

In this study, we propose delayed homomorphic reinforcement learning (DHRL), an RL framework grounded in MDP homomorphisms that defines a belief-equivalence relation over the augmented state space to collapse control-redundant augmented states into a compact abstract state space \citep{symmetry, approximate-MDP}. This framework yields exact abstraction under deterministic dynamics and approximate abstraction under stochastic dynamics, enabling both the actor and critic to benefit from the resulting abstraction. In finite domains, exact abstraction preserves optimality and recovers the sample-complexity order of the corresponding delay-free environment, while approximate abstraction admits a bound on the value loss of the resulting policy. For continuous domains, we introduce deep delayed homomorphic policy gradient (D$^2$HPG), a deep actor-critic instantiation of the DHRL framework based on the stochastic homomorphic policy gradient theorem \citep{HPG2}. Empirical results on diverse continuous-control tasks in MuJoCo \citep{mujoco} demonstrate that our D$^2$HPG outperforms strong augmentation-based baselines, particularly in long-delay regimes.

\section{Preliminaries}

\subsection{Delay-free Reinforcement Learning}

A finite MDP can be defined as $\mathcal{M} = (\mathcal{S}$, $\mathcal{A}$, $\Psi$, $\mathcal{P}$, $\mathcal{R})$, where $\mathcal{S}$ and $\mathcal{A}$ are the finite set of states and actions, $\Psi \subseteq \mathcal{S} \times \mathcal{A}$ is the set of admissible state-action pairs, $\mathcal{P}: \Psi \times \mathcal{S} \rightarrow [0, 1]$ is the transition kernel, and $\mathcal{R}: \Psi \rightarrow \mathbb{R}$ is the reward function. In addition, a policy $\pi: \Psi \rightarrow [0,1]$ specifies an action distribution for each state. For the proposed abstraction framework, we assume a state-independent action set, i.e., $\mathcal{A}_s = \mathcal{A}$ for all $s \in \mathcal{S}$, where $\mathcal{A}_s = \{a \mid (s, a) \in \Psi\}$.

At each time step $t$, the RL agent observes a state $s_t \in \mathcal{S}$ from the environment, selects an action $a_t \in \mathcal{A}$ according to $\pi$, receives a reward $r_t = \mathcal{R}(s_t, a_t)$, and then observes the next state $s_{t+1} \in \mathcal{S}$. The agent repeats this process to find an optimal policy $\pi^{*}$ that maximizes the expected return. Then the value functions are defined as
\begin{align}
\nonumber
V^{\pi}_\mathcal{M}(s) := \mathbb{E}_{\pi}\left[\sum_{k=0}^{\infty}\gamma^{k}r_{t+k}  \mid s_{t}=s \right], \; Q^{\pi}_\mathcal{M}(s, a) :=  \mathbb{E}_{\pi}\left[\sum_{k=0}^{\infty}\gamma^{k}r_{t+k}  \mid s_{t}=s,~a_{t}=a\right],
\end{align} where $\gamma \in [0, 1)$ is a discount factor, $V_\mathcal{M}^{\pi}(s)$ is a state-value function on $\mathcal{M}$ denoting the expected return starting from the state $s$ under the policy $\pi$, and $Q_\mathcal{M}^{\pi}(s, a)$ is an action-value function on $\mathcal{M}$ denoting the expected return starting from the state $s$, taking the action $a$, and thereafter following $\pi$. 

\subsection{Delayed Reinforcement Learning}

We consider a delayed MDP $\mathcal{M}^+ = (\mathcal{M}, \Delta)$ with a fixed observation delay $\Delta \in \mathbb{N}$. Specifically, at each time step $t > \Delta$, the agent observes the delayed state $s_{t-\Delta}$ rather than the current state $s_t$. This MDP can be reformulated as a regular MDP $\mathcal{M}_\Delta = (\mathcal{X}$, $\mathcal{A}$, $\Psi_\Delta$, $\mathcal{P}_\Delta$, $\mathcal{R}_\Delta)$, where $\mathcal{X} = \mathcal{S} \times \mathcal{A}^{\Delta}$ is the finite set of augmented states, $\Psi_\Delta \subseteq \mathcal{X} \times \mathcal{A}$ is the set of admissible augmented state-action pairs, and $\mathcal{P}_\Delta: \Psi_\Delta \times \mathcal{X} \rightarrow [0, 1]$ and $\mathcal{R}_\Delta: \Psi_\Delta \rightarrow \mathbb{R}$ are the augmented transition kernel and augmented reward function. In addition, a regular policy $\pi_\Delta: \Psi_\Delta \rightarrow [0, 1]$ specifies an action distribution for each augmented state. The action set of the regular MDP is likewise assumed to be state-independent, i.e., $\mathcal{A}_x = \mathcal{A}$ for all $x \in \mathcal{X}$, where $\mathcal{A}_x = \{a \mid (x, a) \in \Psi_\Delta\}$.

In the regular MDP, the augmented state $ x_t \in \mathcal{X}$ is defined for all $t > \Delta$ as 
\begin{align}
 \nonumber
 x_{t} := (s_{t-\Delta}, a_{t-\Delta}, a_{t-\Delta+1}, \dots, a_{t-1}),
\end{align}which is composed of the most recently observed state and the most recent $\Delta$ actions. The augmented transition kernel is then defined as
\begin{align}
    \nonumber
    \mathcal{P}_\Delta(x_{t+1} \mid x_t, a_t) := \mathcal{P}(s_{t-\Delta+1} \mid s_{t-\Delta}, a_{t-\Delta})\delta_{a_t}(a'_t)\prod^{\Delta-1}_{i=1}\delta_{a_{t-i}}(a'_{t-i}),
\end{align} where $\delta$ denotes the Kronecker-delta function. Given $x_{t}$, the state $s_t$ can be inferred through a belief
\begin{align}
\nonumber
b_\Delta(s_t \mid x_t) := \int_{\mathcal{S}^{\Delta-1}}\mathcal{P}(s_t \mid s_{t-1}, a_{t-1}) \prod^{t-2}_{i=t-\Delta}\mathcal{P}(s_{i+1} \mid s_{i}, a_{i})ds_{i+1},
\end{align} which represents the probability of being in $s_t$ given $x_t$. In this formulation, the reward feedback is assumed to arrive concurrently with the state feedback to prevent the agent from extracting partial state information. Thus, the augmented reward function is defined as the expectation under the belief:
\begin{align}
\nonumber
\mathcal{R}_\Delta(x_t, a_t) :=  \mathbb{E}_{b_\Delta(\cdot \mid x_t)}\left[\mathcal{R}(s_t, a_t)\right],    
\end{align} where $\tilde{r}_t = \mathcal{R}_\Delta(x_t, a_t)$. Importantly, the regular MDP provides an optimality-preserving Markovian reformulation of the delayed MDP, thereby enabling the direct application of standard RL algorithms without requiring delay-specific algorithmic modifications \citep{augmented}.

\subsection{Finite MDP Homomorphism}

A finite MDP homomorphism \citep{symmetry, approximate-MDP} is a surjection from an MDP onto an abstract MDP that preserves its dynamics while collapsing control-redundant state-action distinctions. Formally, it is defined as follows.

\begin{definition} \label{definition2}
A finite MDP homomorphism from an MDP $\mathcal{M} = (\mathcal{S}, \mathcal{A}, \Psi, \mathcal{P}, \mathcal{R})$ onto an abstract MDP $\bar{\mathcal{M}} = (\bar{\mathcal{S}}, \bar{\mathcal{A}}, \bar{\Psi}, \bar{\mathcal{P}}, \bar{\mathcal{R}})$ is a pair $h_s= \big(f_s, g_s)$ of surjective maps, where $f_s: \mathcal{S} \rightarrow \bar{\mathcal{S}}$ and $g_s: \mathcal{A}_s \rightarrow \bar{\mathcal{A}}_{f_s(s)}$ satisfy
\begin{align}
     \nonumber
     \bar{\mathcal{R}}\left(f_s(s), g_s(a)\right) = \mathcal{R}(s, a), \;\; \bar{\mathcal{P}}\left(f_s(s') \mid f_s(s), g_s(a)\right) = \sum_{s'' \in G_{s}}\mathcal{P}(s'' \mid s, a),  
\end{align} for all $s, s' \in \mathcal{S}, a \in \mathcal{A}_s$, where $G_{s}$ is the block in $B|\mathcal{S}$ to which $s'$ belongs, $B$ is the partition of $\Psi$ induced by the equivalence relation under $h_s$, and $B|\mathcal{S}$ is the projection of $B$ onto $\mathcal{S}$. Here, the abstract MDP $\bar{\mathcal{M}}$ is called the homomorphic image of the original MDP $\mathcal{M}$.
\end{definition} 

The finite MDP homomorphism yields an optimal-value equivalence and preserves optimal policies. In particular, it satisfies $Q_\mathcal{M}^{*}(s,a) = Q_{\bar{\mathcal{M}}}^{*}\big(f_s(s),g_s(a)\big)$ for any pair $(s,a) \in \Psi$, and an analogous equivalence holds for the state-value functions. Moreover, any policy $\bar{\pi}$ on $\bar{\mathcal{M}}$ can be lifted to $\mathcal{M}$ via
\begin{align}
    \nonumber
    \pi^{\uparrow}(a \mid s) = \frac{\bar{\pi}\big(\bar{a} \mid f_s(s) \big)}{|g^{-1}_s(\bar{a})|}, 
\end{align}for any $s \in \mathcal{S}$, $a \in g^{-1}_s(\bar{a})$, and $\bar{a} \in \bar{\mathcal{A}}_{f_s(s)}$, where $g^{-1}_s(\bar{a})$ denotes the set of actions that have the same image $\bar{a}$ under $g_s$ and $\pi^{\uparrow}$ denotes the lifted policy on $\mathcal{M}$. Crucially, this policy-lifting process preserves optimality, thereby justifying learning policies in the abstract MDP. 

As established by \citet{symmetry}, if a partition $B$ of $\Psi$ is reward-respecting and satisfies the stochastic substitution property (SSP), then the quotient MDP $\mathcal{M}/B$ can be constructed, and there exists a homomorphism $h_s:\mathcal{M}\to \mathcal{M}/B$. This implies that an appropriately chosen partition $B$ can yield a quotient MDP that provides a useful abstraction of $\mathcal{M}$. For efficient policy learning, we thus seek a reward-respecting SSP partition $B$ of $\Psi$ that induces a compact abstract MDP $\bar{\mathcal{M}}:= \mathcal{M}/B$.

\section{Delayed Homomorphic Reinforcement Learning} \label{section3}

Despite its solid theoretical foundation, state augmentation inevitably
increases the sample-complexity burden as the augmented state space $\mathcal{X} = \mathcal{S} \times \mathcal{A}^{\Delta}$ grows exponentially in $\Delta$. To address this limitation, we propose delayed homomorphic reinforcement learning (DHRL), a framework grounded in MDP homomorphisms for sample-efficient learning on regular MDPs. In this section, we define a belief-equivalence relation and show that it induces exact abstraction over the augmented state space under deterministic dynamics, while its relaxation yields approximate abstraction under stochastic dynamics. For brevity, detailed proofs of the theoretical results in this section are deferred to Appendix~\ref{Appendix-A}.

\subsection{Exact Abstraction under Belief-equivalence} \label{exact}

We first consider deterministic dynamics, under which exact abstraction of the regular MDP becomes possible. This setting serves as a special case of the more general stochastic formulation presented in the following subsection. We formalize the finite MDP homomorphism for delayed RL as follows.

\begin{definition}\label{def3.1}
A finite MDP homomorphism from a regular MDP $\mathcal{M}_\Delta = (\mathcal{X}$, $\mathcal{A}$, $\Psi_\Delta$, $\mathcal{P}_\Delta$, $\mathcal{R}_\Delta)$ onto an abstract MDP $\bar{\mathcal{M}}_\Delta = (\bar{\mathcal{X}}, \bar{\mathcal{A}}, \bar{\Psi}_\Delta, \bar{\mathcal{P}}_\Delta, \bar{\mathcal{R}}_\Delta)$ is a pair $h_x = (f_x, g_x)$ of surjective maps, where $f_x:\mathcal{X}\to \bar{\mathcal{X}}$ and $g_x:\mathcal{A}_x\to \bar{\mathcal{A}}_{f_x(x)}$ satisfy
\begin{align}
    \nonumber
    \bar{\mathcal{R}}_\Delta\left(f_x(x), g_x(a)\right) = \mathcal{R}_\Delta(x, a), \;\; 
    \bar{\mathcal{P}}_\Delta\left(f_x(x') \mid f_x(x), g_x(a)\right) = \sum_{x'' \in G_x}\mathcal{P}_\Delta(x'' \mid x, a),
\end{align}
for all $x, x' \in \mathcal{X}$ and $a \in \mathcal{A}_x$, where $G_x$ is the block in $B_\Delta|\mathcal{X}$ that contains $x'$, $B_\Delta$ is the partition of $\Psi_\Delta$ induced by the equivalence relation under $h_x$, and $B_\Delta|\mathcal{X}$ denotes the projection of $B_\Delta$ onto $\mathcal{X}$. Here, the abstract MDP $\bar{\mathcal{M}}_\Delta$ is called the homomorphic image of the regular MDP $\mathcal{M}_\Delta$.
\end{definition}

To address the state-space explosion issue, we seek a compact abstraction of the regular MDP by identifying a reward-respecting SSP partition $B_\Delta$ of $\Psi_\Delta$ that induces the abstract MDP $\bar{\mathcal{M}}_\Delta:= \mathcal{M}_\Delta/B_\Delta$. To this end, we introduce a belief-equivalence relation over the augmented state space $\mathcal{X}$.

\begin{definition}[Belief-equivalence]\label{b-equivalent}
Let $x, x' \in \mathcal{X}$ be augmented states. We say that $x$ and $x'$ are belief-equivalent if they induce the same belief over the underlying state space $\mathcal{S}$, i.e.,
\begin{align}
    x \equiv_{b_\Delta} x'
    \quad \Longleftrightarrow \quad
    b_\Delta(\cdot \mid x) = b_\Delta(\cdot \mid x').
    \nonumber
\end{align}
\end{definition}

Intuitively, if two augmented states induce the same belief over the underlying state, then distinguishing between them provides no additional control-relevant information. This suggests that they may be merged into the same abstract state. Based on this relation, the following proposition holds.

\begin{proposition} \label{proposition-1}
Let $B_\Delta$ be the partition of $\Psi_\Delta$ induced by the belief-equivalence relation on $\mathcal{X}$. If the dynamics are deterministic, then the partition $B_\Delta$ is reward-respecting and satisfies SSP.
\end{proposition}


Proposition~\ref{proposition-1} implies that the homomorphic image $\bar{\mathcal{M}}_\Delta$ can be constructed from the belief-induced partition $B_\Delta$ of $\Psi_\Delta$. Moreover, the following corollary shows that an optimal policy of $\bar{\mathcal{M}}_\Delta$ can be lifted back to the delayed MDP $\mathcal{M}^+$ without loss of optimality.

\begin{corollary}\label{corollary-1}
Let $\mathcal{M}^+$ be a delayed MDP, $\mathcal{M}_\Delta$ be
its regular reformulation, and $\bar{\mathcal{M}}_\Delta$ be the
homomorphic image of $\mathcal{M}_\Delta$ induced by the belief-equivalence relation. If the dynamics are deterministic, then an optimal policy on $\bar{\mathcal{M}}_\Delta$ can be lifted back to $\mathcal{M}^+$ without loss of optimality.
\end{corollary}

\begin{proof}[Proof sketch]
By Theorem~2 in \citet{symmetry}, an optimal policy on the homomorphic
image $\bar{\mathcal{M}}_\Delta$ remains optimal when lifted back to the
regular MDP $\mathcal{M}_\Delta$. Since $\mathcal{M}_\Delta$ is an
optimality-preserving Markovian reformulation of the delayed MDP
$\mathcal{M}^+$ \citep{augmented}, the lifted policy remains optimal on
$\mathcal{M}^+$. Thus, the corollary follows immediately.
\end{proof}

\begin{figure}[h] 
\begin{center}
\includegraphics[width=1.0\textwidth]{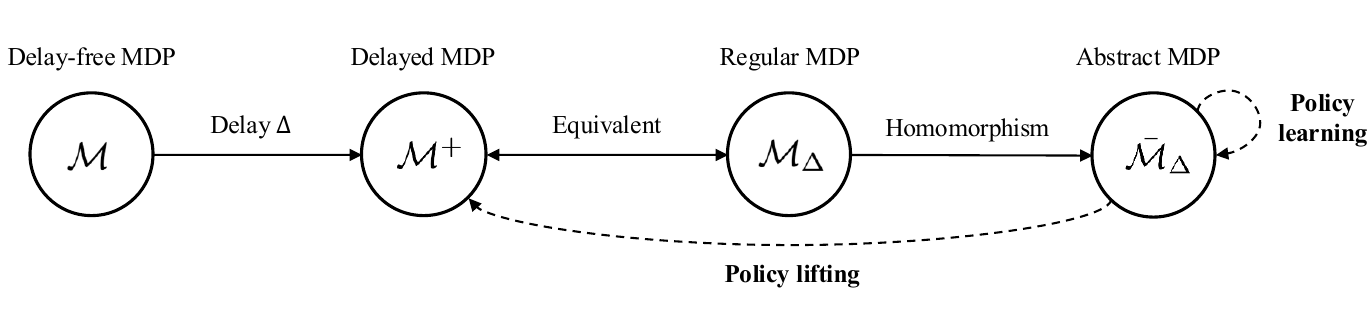}
\end{center}
\caption{Schematic overview of the DHRL framework.} 
\label{fig:framework0}
\end{figure}

Corollary~\ref{corollary-1} implies that an optimal policy for the delayed MDP $\mathcal{M}^+$ can be obtained by solving the abstract MDP $\bar{\mathcal{M}}_\Delta$ and lifting the resulting policy back to $\mathcal{M}^+$ without loss of optimality. A schematic overview of this framework is presented in Fig.~\ref{fig:framework0}. We next analyze the state-space reduction induced by the belief-equivalence relation and its implications for the sample-complexity bound.

\begin{proposition} \label{proposition-2}
The compression ratio of the abstract state space $\bar{\mathcal{X}}$ induced by belief-equivalence on the augmented state space $\mathcal{X}$ satisfies, under deterministic dynamics,
\begin{align}
    \nonumber
    \zeta := \frac{|\bar{\mathcal{X}}|}{|\mathcal{X}|} 
    \leq \frac{1}{|\mathcal{A}|^{\Delta}},
\end{align}
for any $\Delta \in \mathbb{N}$. In particular, this implies that $|\bar{\mathcal{X}}| \leq |\mathcal{S}|$.
\end{proposition}


Proposition~\ref{proposition-2} implies that the abstract state space $\bar{\mathcal{X}}$ induced by the belief-equivalence relation is no larger than the underlying state space $\mathcal{S}$, thereby yielding the following corollary.

\begin{corollary} \label{corollary-2}
Under deterministic dynamics, the sample-complexity of $Q$-learning on the abstract MDP $\bar{\mathcal{M}}_\Delta$ is upper-bounded by 
\begin{align}
        \nonumber
        O\Big(\frac{\log (\vert \mathcal{S}\vert\vert\mathcal{A} \vert)}{\varepsilon^{2.5}(1 - \gamma)^5}\Big),
\end{align} which is of the same order as the sample-complexity bound for the delay-free MDP.    
\end{corollary}


The sample-complexity bound characterizes the number of samples required for $Q$-learning to attain an $\varepsilon$-optimal policy with high confidence \citep{Speed-QL}. Corollary~\ref{corollary-2} shows that the sample-complexity bound on the abstract MDP does not explicitly depend on the delay $\Delta$, unlike that on the regular MDP defined over the augmented state space $\mathcal{X}=\mathcal{S}\times\mathcal{A}^{\Delta}$. This result implies that the proposed abstraction recovers the delay-free sample-complexity order. Consequently, policy learning can be performed on the abstract MDP in a sample-efficient manner, and the resulting policy can be lifted back to the delayed MDP without loss of optimality, as guaranteed by Corollary~\ref{corollary-1}.

\subsection{Approximate Abstraction under Relaxed Belief-equivalence}

Under general stochastic dynamics, exact belief-equivalence in Definition~\ref{b-equivalent} rarely holds. Therefore, we introduce a relaxed belief-based state-aggregation criterion and analyze the performance degradation induced by the resulting approximate abstraction. Specifically, for $\varepsilon \in [0,1)$, we consider an $\varepsilon$-abstract partition of the augmented state space $\mathcal{X}$ such that any two augmented states $x,x' \in \mathcal{X}$ belonging to the same abstract block satisfy 
\begin{align}
    \nonumber
    \|b_\Delta(\cdot \mid x) - b_\Delta(\cdot \mid x')\|_{\mathrm{TV}} \leq \varepsilon,
\end{align} 
where $\|\cdot\|_{\mathrm{TV}}$ denotes the total variation distance. 

Let $\bar{\mathcal{X}}_\varepsilon$ denote the resulting abstract state space, $\bar{\mathcal{M}}_{\varepsilon} = (\bar{\mathcal{X}}_\varepsilon,\bar{\mathcal{A}}, \bar{\Psi}_{\varepsilon}, \bar{\mathcal{P}}_{\varepsilon}, \bar{\mathcal{R}}_{\varepsilon})$ be the corresponding approximate abstract MDP, and $f_x:\mathcal{X}\to\bar{\mathcal{X}}_\varepsilon$ and $g_x:\mathcal{A}_x\to\bar{\mathcal{A}}_{f_x(x)}$ be the associated surjective state and action maps. For brevity, we defer the full definition of the approximate abstract MDP to Appendix~\ref{approximate-abstract}. Since this relaxation does not generally yield an exact reward-respecting SSP partition, the optimality guarantee in Corollary~\ref{corollary-1} no longer applies. Instead, the suboptimality of the lifted policy can be quantified through the value-loss bound established in \citet{approximate-MDP}.

\begin{lemma} \label{lemma:approx_perf}
Let $\bar{\pi}_{\varepsilon}^*$ be an optimal policy on the approximate abstract MDP $\bar{\mathcal{M}}_{\varepsilon}$ and let $(\bar{\pi}_{\varepsilon}^*)^\uparrow$ be its lifted policy on the regular MDP $\mathcal{M}_\Delta$. Then the value loss of the lifted policy satisfies
\begin{align}
\nonumber
\left\|V_{\mathcal{M}_\Delta}^* - V_{\mathcal{M}_\Delta}^{(\bar{\pi}_{\varepsilon}^*)^\uparrow}\right\|_\infty \leq \frac{2}{1-\gamma} \left(C_1 + \frac{\gamma\xi_{\varepsilon}}{2(1-\gamma)}C_2\right),
\end{align}
where 
\begin{align}    
    \nonumber
    C_1
    &=
    \max_{(x,a) \in \Psi_\Delta}
    \left|\mathcal{R}_\Delta(x,a)-\bar{\mathcal{R}}_{\varepsilon}(f_x(x),g_x(a))\right|, \\ \nonumber
    C_2
    &=
    \max_{(x,a) \in \Psi_\Delta}\sum_{\bar{x}' \in \bar{\mathcal{X}}_\varepsilon}
    \left|\sum_{x'' : f_x(x'') = \bar{x}'}\mathcal{P}_\Delta(x'' \mid x,a)-\bar{\mathcal{P}}_{\varepsilon}(\bar{x}' \mid f_x(x),g_x(a))\right|,\\ 
    \nonumber
    \xi_{\varepsilon} &= \max_{(\bar{x}, \bar{a}) \in \bar{\Psi}_\varepsilon} \bar{\mathcal{R}}_{\varepsilon}(\bar{x}, \bar{a}) - \min_{(\bar{x}, \bar{a}) \in \bar{\Psi}_\varepsilon} \bar{\mathcal{R}}_{\varepsilon}(\bar{x}, \bar{a}).
\end{align} Here, $C_1$ and $C_2$ denote the maximum reward and transition discrepancies between $\mathcal{M}_\Delta$ and $\bar{\mathcal{M}}_\varepsilon$, respectively, and $\xi_{\varepsilon}$ denotes the range of the abstract reward function.
\end{lemma}

Lemma~\ref{lemma:approx_perf} follows directly by applying the value-loss bound formalized in \citet{approximate-MDP} to the regular MDP formulation. This implies that when the reward and transition discrepancies are small, an optimal policy on the approximate abstract MDP can be lifted to the regular MDP with bounded performance loss. Given a policy $\bar{\pi}_\varepsilon$ on $\bar{\mathcal{M}}_{\varepsilon}$, its lifted policy on $\mathcal{M}_\Delta$ is defined as
\begin{align}
    \nonumber
    (\bar{\pi}_\varepsilon)^\uparrow(a \mid x) = \frac{\bar{\pi}_\varepsilon(\bar{a}\mid f_x(x))}{|g_x^{-1}(\bar{a})|}, 
\end{align} for any $a \in g_x^{-1}(\bar{a})$, which takes the same form as in the exact abstraction case. Unlike in the exact case, however, lifting an optimal policy from $\bar{\mathcal{M}}_{\varepsilon}$ does not generally preserve optimality in $\mathcal{M}_\Delta$. Instead, the performance loss of the lifted policy is quantified by the value-loss bound in Lemma~\ref{lemma:approx_perf}. 

\section{Deep Delayed Homomorphic Policy Gradient}

We have presented the DHRL framework for finite domains based on the belief-equivalence relation and finite MDP homomorphisms, in which the abstract MDP can be constructed through an explicit partition and quotient construction. In continuous domains, however, extending this abstraction principle would require identifying belief-equivalent augmented states in continuous spaces, which is intractable in general. To this end, we employ the continuous MDP homomorphism framework and the stochastic homomorphic policy gradient theorem established in \citet{HPG, HPG2} to extend our finite-domain formulation to continuous domains.


\citet{HPG2} extends finite MDP homomorphisms to continuous domains and shows that the policy gradient in the original MDP can be obtained from its abstract MDP. This leads to the stochastic homomorphic policy gradient theorem, which we specialize to the delayed setting below.

\begin{lemma}\label{SHPG}
Let $\bar{\mathcal{M}}_\Delta$ be the homomorphic image of the regular MDP $\mathcal{M}_\Delta$ under a continuous MDP homomorphism $h_x=(f_x,g_x)$, where the augmented state space and the action space are continuous. Let $\pi_\theta^\uparrow$ and $\bar{\pi}_\theta$ be the stochastic policies defined on $\mathcal{M}_\Delta$ and $\bar{\mathcal{M}}_\Delta$, respectively, where $\pi_\theta^\uparrow$ is the regular policy induced by the abstract policy $\bar{\pi}_\theta$ via policy lifting. Then the policy gradient of the performance measure $J_{\mathcal{M}_\Delta}(\theta)$ with respect to $\theta$ is given by
    \begin{align}
        \nabla_\theta J_{\mathcal{M}_\Delta}(\theta) = \int_{\bar{x} \in \bar{\mathcal{X}}} \rho^{\bar{\pi}_\theta}(\bar{x})\int_{\bar{a} \in \bar{\mathcal{A}}}Q^{\bar{\pi}_\theta}_{\bar{\mathcal{M}}_\Delta}(\bar{x}, \bar{a}) \nabla_\theta \bar{\pi}_\theta(d\bar{a} \mid \bar{x}) d\bar{x},
    \end{align}where $\rho^{\bar{\pi}_\theta}$ denotes the discounted stationary distribution on $\bar{\mathcal{M}}_\Delta$ induced by $\bar{\pi}_\theta$.
\end{lemma} 

Lemma~\ref{SHPG} implies that the regular policy $\pi_\theta^\uparrow$ defined on the regular MDP can be optimized using the policy gradient obtained from the abstract MDP, which is referred to as the homomorphic policy gradient. Based on this result, we propose a deep actor-critic instantiation of the DHRL framework for continuous domains, termed deep delayed homomorphic policy gradient (D$^2$HPG). The learning pipeline closely resembles that of \citet{HPG2}, yet admits a simplified instantiation under a mild assumption. The corresponding pseudo-code for D$^2$HPG is provided in Appendix~\ref{pseudo-code}.

The following subsections first describe the principled formulation of D$^2$HPG and then present its simplified practical instantiation, which constitutes the core contribution of this work. 

\subsection{Principled Formulation of D$^2$HPG}

In the principled formulation of D$^2$HPG, belief-equivalence provides the theoretical basis for constructing the abstract MDP. Since directly identifying belief-equivalent augmented states is generally intractable in continuous domains, we instead learn the abstraction by promoting reward and transition consistency through a lax-bisimulation objective \citep{lax-bisimulation-metric, HPG2}.

In particular, the homomorphism map $h_{\xi,\omega}=(f_\xi,g_\omega)$ is learned by minimizing the lax-bisimulation loss $\mathcal{L}_{\text{lax}}$ using pairs of transitions $\{(x_t,a_t,x_{t+1},\tilde{r}_t)\}_{t=i,j}$ drawn from a replay buffer $\mathcal{D}$:
\begin{align}
    \mathcal{L}_{\text{lax}}(\xi, \omega)
    =
    \mathbb{E}_{\mathcal{D}}
    \left[
    \left(
    \Vert \bar{x}^+_i - \bar{x}^+_j \Vert_1
    -
    \left(
    \vert \tilde{r}_i - \tilde{r}_j \vert
    +
    \gamma W_2
    \left(
    \bar{\mathcal{P}}_{\tau}(\cdot \mid \bar{x}^+_i,\bar{a}^+_i),
    \bar{\mathcal{P}}_{\tau}(\cdot \mid \bar{x}^+_j,\bar{a}^+_j)
    \right)
    \right)
    \right)^2
    \right],
    \label{learning1}
\end{align}
where $\bar{x}^+_t=f_\xi(x_t)$ and $\bar{a}^+_t=g_\omega(x_t,a_t)$ are the predicted abstract state and action, respectively, and $W_2$ denotes the 2-Wasserstein distance between the Gaussian abstract transition distributions modeled by $\bar{\mathcal{P}}_\tau$. To capture the reward and transition behavior required for learning the abstraction, the abstract reward model $\bar{\mathcal{R}}_\nu$ and transition model $\bar{\mathcal{P}}_\tau$ are trained using the auxiliary loss $\mathcal{L}_{\text{h}}$:
\begin{align}
    \mathcal{L}_{\text{h}}(\tau, \nu)
    =
    \mathbb{E}_{\mathcal{D}}
    \left[
    \Vert \bar{x}^+_{i+1}-\bar{x}_{i+1}\Vert_2^2
    +
    (\bar{r}^+_i-\tilde{r}_i)^2
    \right],
    \label{learning2}
\end{align}
where $\bar{x}_{i+1}\sim \bar{\mathcal{P}}_\tau(\cdot\mid\bar{x}^+_i,\bar{a}^+_i)$ and $\bar{r}^+_i=\bar{\mathcal{R}}_\nu(\bar{x}^+_i,\bar{a}^+_i)$. 


$\mathcal{L}_{\text{lax}}$ and $\mathcal{L}_{\text{h}}$ are jointly minimized to learn the homomorphism map and the corresponding abstract reward and transition models. The resulting abstraction components are then applied to transition samples obtained from the regular MDP to construct abstract transitions. These abstract transitions are then used to train the abstract actor and critic in a model-free manner.

By Lemma~\ref{SHPG}, the regular policy $\pi^\uparrow_\theta$ can be optimized solely using the homomorphic policy gradient obtained from the abstract MDP. In practice, however, \citet{HPG2} suggest training the regular policy with a stochastic actor-critic algorithm while optimizing the abstract policy using the homomorphic policy gradient. Since the two policies are optimized separately, they must be coupled through the policy-lifting process to transfer the information learned by the abstract policy to the regular policy. In continuous domains, the coupling can be approximated through sampling-based policy alignment \citep{approx-lifting}.

Specifically, let the regular and abstract policies $\pi^{\uparrow}_\theta$ and $\bar{\pi}_{\theta'}$ be parameterized by independent neural networks. The lifting loss $\mathcal{L}_\text{lift}$ is introduced to approximate the policy-lifting relation by matching the conditional expectation and standard deviation of the abstract actions induced by the two policies:
\begin{align} \label{approx-policy-lift}
    \mathcal{L}_\text{lift}(\theta, \theta') = \left\Vert \mathbb{E}_{\pi^{\uparrow}_\theta(\cdot \mid x_i)}\left[\bar{a}^+\right] - \mathbb{E}_{\bar{\pi}_{\theta'}(\cdot \mid \bar{x}^+_i)}\left[\bar{a}\right]\right\Vert_2^2 + \left\Vert\sigma_{\pi^{\uparrow}_\theta(\cdot\mid x_i)}\left[\bar{a}^+\right] - \sigma_{\bar{\pi}_{\theta'}(\cdot\mid \bar{x}^+_i)}\left[\bar{a}\right]\right\Vert_2^2, 
\end{align} where $\bar{a}^+ = g_\omega(x_i, a)$ and $\sigma[\cdot]$ denotes the standard deviation. By minimizing this loss, the regular and abstract policies are aligned so that the gradient information obtained from the abstract MDP can be transferred to the regular policy. The critic losses can be written in a generic one-step TD form as
\begin{align}
    \nonumber
    \mathcal{L}_\text{regular-critic}(\phi_1) &= \mathbb{E}_{\mathcal{D}}\big[(\tilde{r}_i+ \gamma Q_{\phi_1'}(x_{i+1}, a_{i+1}) - Q_{\phi_1}(x_i, a_i))^2\big],\\ \nonumber
    \mathcal{L}_\text{abstract-critic}(\phi_2) &= \mathbb{E}_{\mathcal{D}}\big[({\bar{r}}^{+}_i + \gamma Q_{\phi_2'}(\bar{x}^{+}_{i+1}, \bar{a}^{+}_{i+1}) - Q_{\phi_2}(\bar{x}^{+}_i, \bar{a}^{+}_i))^2\big],
\end{align} where $\phi_1'$ and $\phi_2'$ denote the parameters of the target networks used to stabilize critic learning. {The target terms can be modified according to the chosen actor-critic algorithm.

\subsection{Practical Instantiation of D$^2$HPG} \label{practical}

\begin{figure}[h] 
\begin{center}
\includegraphics[width=0.5\textwidth]{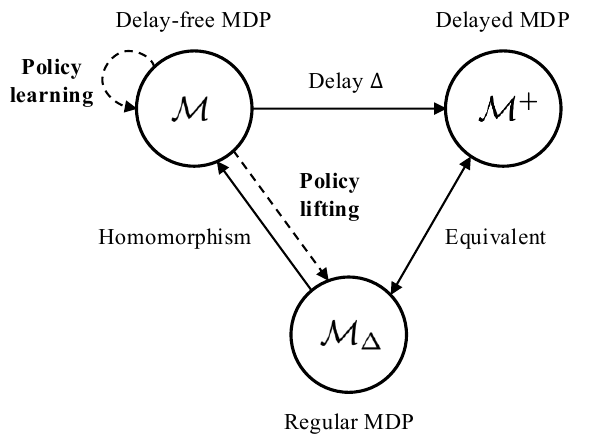}
\end{center}
\caption{Schematic overview of practical D$^2$HPG.} 
\label{fig:framework}
\end{figure}

In the principled formulation of D$^2$HPG, the abstraction components must be learned from transition samples collected in the regular MDP. However, discovering control-redundant structure directly in its augmented state space $\mathcal{X}=\mathcal{S} \times \mathcal{A}^{\Delta}$ may require substantial additional transition data and computational effort, potentially offsetting the sample-efficiency gains that the learned abstraction is intended to provide. Motivated by this practical limitation, we exploit the homomorphic structure revealed by the belief-equivalence relation to propose a simplified instantiation of D$^2$HPG, termed practical D$^2$HPG. The following proposition provides the theoretical basis for this simplification.

\begin{proposition}\label{prop:practical_exact}
Under deterministic dynamics, the delay-free MDP coincides with the
homomorphic image of the regular MDP. Consequently, the abstract policy
can be identified with the delay-free policy, eliminating the need to
learn the homomorphism map and abstract dynamics models.
\end{proposition}

\begin{proof}
    See Appendix~\ref{proposition_4_2}.
\end{proof}

Proposition~\ref{prop:practical_exact} justifies using the delay-free MDP as the abstract MDP under deterministic dynamics, eliminating the need to learn the abstraction components in Eqs.~\eqref{learning1}--\eqref{learning2}. In practice, delay-free transitions $(s_{t-\Delta}, a_{t-\Delta}, r_{t-\Delta}, s_{t-\Delta+1})$ can be reconstructed from samples collected in the regular MDP through retrospective time alignment, without direct interaction with the delay-free MDP. These transition samples are used to train the delay-free actor and critic in a model-free manner. The regular policy is then aligned with the delay-free policy by minimizing the lifting loss in Eq.~\eqref{approx-policy-lift}. Notably, this simplified approach can remain effective under moderate stochasticity, as supported by the empirical results in Section~\ref{stochasticity}. A schematic overview of practical D$^2$HPG is provided in Fig.~\ref{fig:framework}.

\section{Experiments} \label{experiment}

\subsection{Performance Evaluation}

We evaluate practical D$^2$HPG on MuJoCo benchmark against representative delayed RL baselines: naive SAC~\citep{sac}, augmented SAC~\citep{augmented}, delayed SAC~\citep{acting}, BPQL~\citep{BPQL}, and VDPO~\citep{VDPO}. Naive SAC ignores delays, while delayed SAC recursively predicts unobserved states using an approximate dynamics model. Augmented SAC, BPQL, and VDPO are augmentation-based baselines that differ in how they learn the regular policy: augmented SAC directly learns the regular policy by applying SAC to the augmented state space, BPQL employs an alternative representation of augmented-state value functions to improve sample efficiency, and VDPO distills a delay-free policy into the regular policy via behavior cloning. In our implementation of practical D$^2$HPG, we employ BPQL as the regular-policy learner. Hereafter, D$^2$HPG refers to its practical variant unless otherwise specified.


\begin{table*}[!h] 
\caption{Results on continuous-control MuJoCo tasks with fixed delays $\Delta \in \{5, 10, 20\}$. Returns are averaged over evaluations throughout one million time steps across five random seeds, with standard deviations indicated by $\pm$. The best results are highlighted in {\color{blue}blue}.}
\label{result-table}
\centering
\renewcommand{\arraystretch}{1.576}
\resizebox{\textwidth}{!}{
\begin{tabular}{clcccccc}
\hline
\multicolumn{2}{c}{\textbf{Environment}} & \multirow{2}{*}{Ant-v3} &  \multirow{2}{*}{HalfCheetah-v3} & \multirow{2}{*}{Hopper-v3} & \multirow{2}{*}{Walker2d-v3} & \multirow{2}{*}{Humanoid-v3} & \multirow{2}{*}{InvertedPendulum-v2} \\ 

\cline{1-2} 
$\Delta$ & \textbf{Algorithm}  & \\

\hline
0  & Delay-free SAC             & $3279.2_{\pm180}$ & $8608.4_{\pm57}$ & $2435.2_{\pm23}$ & $3305.5_{\pm234}$ & $3228.1_{\pm410}$ & $964.3_{\pm29}$ \\
\hline
\multirow{6}{*}{5}  & Naive SAC & $-74.9_{\pm4}$ & $-276.3_{\pm5}$ & $88.8_{\pm10}$ & $44.5_{\pm20}$ &          $398.3_{\pm6}$ & $32.1_{\pm2}$ \\                                                  
                    & Delayed SAC & $1093.6_{\pm132}$ & $1753.1_{\pm198}$ & $1536.6_{\pm248}$ & $858.8_{\pm207}$ & $505.9_{\pm68}$ & $925.5_{\pm18}$ \\                       
                    & Augmented SAC & $881.9_{\pm103}$ & $2130.9_{\pm344}$ & $2230.8_{\pm178}$ & $1265.4_{\pm303}$ & $629.2_{\pm56}$ & $935.7_{\pm38}$ \\                     
                    & VDPO   & ${\color{blue}\mathbf{4373.3_{\pm181}}}$ & {{$4819.5_{\pm36}$}} & $1917.6_{\pm105}$ & ${\color{blue}\mathbf{3402.9_{\pm263}}}$ & {{$2843.2_{\pm482}$}}  & ${764.5}_{\pm136}$ \\                 
                    & BPQL    & {{$3754.1_{\pm102}$}} &  {{$5216.3_{\pm43}$}} & $2136.3_{\pm158}$ & {{$2477.4_{\pm140}$}} & $3162.9_{\pm276}$ & {${945.5}_{\pm20}$} \\                                           
                    & \textbf{D$^2$HPG} (ours) & ${3852.3}_{\pm198}$ & ${\color{blue}\mathbf{5226.4_{\pm87}}}$ & ${\color{blue}\mathbf{2509.8_{\pm157}}}$ & $2704.7_{\pm328}$ & ${\color{blue}\mathbf{3320.6_{\pm255}}}$ & ${\color{blue}\mathbf{949.8_{\pm6}}}$ \\               
\hline
\multirow{6}{*}{10}  & Naive SAC & $-79.2_{\pm7}$ & $-281.3_{\pm11}$ & $39.7_{\pm6}$ & $59.2_{\pm7}$ & $394.1_{\pm10}$ & $20.7_{\pm1}$ \\                 
                    & Delayed SAC & $924.3_{\pm66}$ & $555.1_{\pm23}$ & $1403.1_{\pm216}$ & $207.5_{\pm29}$ & $341.6_{\pm11}$ & $714.2_{\pm50}$ \\                       
                    & Augmented SAC & $880.7_{\pm30}$ & $946.2_{\pm94}$ & $1002.6_{\pm182}$ & $1335.6_{\pm348}$ & $520.5_{\pm5}$ & $932.2_{\pm15}$ \\                    
                    & VDPO    & ${\color{blue}\mathbf{3085.2_{\pm106}}}$ & {{$3328.5_{\pm184}$}} & {{$1942.3_{\pm114}$}} & ${\color{blue}\mathbf{2588.4_{\pm201}}}$ & {{$2189.4_{\pm474}$}} & $597.3_{\pm110}$ \\                                   
                    & BPQL    & {{$2824.9_{\pm103}$}} & {{$4266.6_{\pm192}$}} & {{$2045.2_{\pm190}$}} & {{$2331.6_{\pm252}$}} & {{$2889.5_{\pm310}$}} & $919.7_{\pm34}$ \\                                  
                    & \textbf{D$^2$HPG} (ours) & $3010.6_{\pm70}$ & ${\color{blue}\mathbf{4551.4_{\pm118}}}$ & ${\color{blue}\mathbf{2374.8_{\pm160}}}$ & $2387.5_{\pm263}$ & ${\color{blue}\mathbf{2996.4_{\pm285}}}$ & ${\color{blue}\mathbf{937.3_{\pm22}}}$ \\               
\hline
\multirow{6}{*}{20} & Naive SAC & $-83.9_{\pm7}$ & $-262.1_{\pm5}$ & $27.6_{\pm5}$ & $54.6_{\pm11}$ & $362.9_{\pm5}$ & $24.3_{\pm2}$ \\                               
                    & Delayed SAC & $817.9_{\pm78}$ & $552.5_{\pm40}$ & $595.5_{\pm90}$ & $102.9_{\pm3}$ & $407.2_{\pm11}$ & $67.7_{\pm10}$ \\                       
                    & Augmented SAC & $697.4_{\pm80}$ & $110.7_{\pm120}$ & $298.6_{\pm21}$ & $347.2_{\pm51}$ & $340.6_{\pm82}$ & $340.7_{\pm84}$ \\                    
                    & VDPO    & $2419.7_{\pm95}$ & {{$2342.3_{\pm164}$}} & {{$1359.5_{\pm165}$}} & {{$795.2_{\pm123}$}} & {{$791.3_{\pm88}$}} & $19.4_{\pm8}$ \\                    
                    & BPQL    & $2069.5_{\pm138}$ & $2861.3_{\pm241}$ & $1526.7_{\pm227}$ & $846.7_{\pm443}$ & $1197.7_{\pm457}$ & $568.1_{\pm79}$ \\                                    
                    & \textbf{D$^2$HPG} (ours) & ${\color{blue}\mathbf{2694.3_{\pm99}}}$ & ${\color{blue}\mathbf{4089.4_{\pm146}}}$ & ${\color{blue}\mathbf{1818.9_{\pm211}}}$ & ${\color{blue}\mathbf{1151.6_{\pm246}}}$ & ${\color{blue}\mathbf{1829.5_{\pm333}}}$ & ${\color{blue}\mathbf{637.3_{\pm53}}}$ \\
\hline
\end{tabular}}
\end{table*}


The results in Table~\ref{result-table} demonstrate that D$^2$HPG achieves strong overall performance, particularly under long delays. Specifically, augmented SAC and delayed SAC struggle even at relatively short delays ($\Delta=5$), likely due to state-space explosion and recursive prediction errors, respectively. Although BPQL and VDPO remain competitive at short to moderate delays ($\Delta \in \{5,10\}$), their performance degrades substantially as the delay increases ($\Delta=20$). In contrast, D$^2$HPG exhibits less performance degradation and a clear advantage at $\Delta=20$. 

The reported returns are averaged over evaluations conducted throughout training, reflecting both learning speed and late-training performance rather than solely final-checkpoint performance. Details on these baselines and the MuJoCo benchmark are provided in Appendices~\ref{related-work} and~\ref{Appendix-detail}, respectively. Additional results, including the full learning curves, are provided in Appendix~\ref{full-results}.

\subsection{Robustness under Stochasticity} \label{stochasticity}
The practical D$^2$HPG employs the delay-free MDP as a surrogate for the abstract MDP. To assess whether this surrogate remains effective beyond the deterministic limit, we evaluate D$^2$HPG and BPQL under actuator-level stochastic perturbations across multiple MuJoCo tasks with delay $\Delta = 10$. Specifically, the action space of each task is normalized to $[-1,1]$, and the executed actions are perturbed by zero-mean Gaussian noise with standard deviation $\sigma \in \{0.01, 0.03, 0.05, 0.07, 0.10\}$, corresponding to 1$\%$, 3$\%$, 5$\%$, 7$\%$, and 10$\%$ of the maximum absolute action magnitude, respectively. Since the perturbation is injected during execution, it induces stochasticity in the transition dynamics. 


\begin{table*}[!h]
\caption{Normalized performance of practical D$^2$HPG and BPQL under actuator-level stochastic perturbations across MuJoCo tasks with delay $\Delta = 10$. Performance is normalized with respect to the noise-free setting ($\sigma = 0.00$) and averaged over evaluations conducted throughout one million time steps across five random seeds. Practical D$^2$HPG is denoted as D$^2$HPG in the table for brevity.}
\label{Tab:experiment-2}
\centering
\renewcommand{\arraystretch}{1.4}
\footnotesize
\begin{tabular*}{\textwidth}{@{\extracolsep{\fill}}llcccccc@{}}
\hline
\textbf{Environment} & \textbf{Algorithm} &  $\sigma = 0.00$ & $\sigma = 0.01$ & $\sigma = 0.03$ & $\sigma = 0.05$ & $\sigma = 0.07$ & $\sigma = 0.10$ \\
\hline
\multirow{2}{*}{Ant-v3}
& BPQL & 1.00 & 0.97 & 0.98 & 0.89 & 0.83 & 0.69 \\
& \textbf{D$^2$HPG} (ours) & 1.00 & 0.98 & 0.96 & 0.94 & 0.90 & 0.78 \\
\hline
\multirow{2}{*}{Hopper-v3}
& BPQL & 1.00 & 1.04 & 1.01 & 1.03 & 0.97 & 0.84 \\
& \textbf{D$^2$HPG} (ours)& 1.00 & 0.98 & 0.98 & 1.00 & 0.85 & 0.68 \\
\hline
\multirow{2}{*}{Walker2d-v3}
& BPQL & 1.00 & 0.99 & 1.04 & 0.98 & 0.79 & 0.72 \\
& \textbf{D$^2$HPG} (ours)& 1.00 & 1.04 & 1.05 & 1.00 & 0.86 & 0.80 \\
\hline
\multirow{2}{*}{Humanoid-v3}
& BPQL & 1.00 & 0.98 & 1.00 & 0.92 & 0.94 & 0.89 \\
& \textbf{D$^2$HPG} (ours)& 1.00 & 1.02 & 0.99 & 0.95 & 0.92 & 0.84 \\
\hline
\multirow{2}{*}{HalfCheetah-v3}
& BPQL & 1.00 & 1.03 & 0.98 & 0.92 & 0.83 & 0.75 \\
& \textbf{D$^2$HPG} (ours)& 1.00 & 0.95 & 0.93 & 0.90 & 0.74 & 0.65 \\
\hline
\end{tabular*}
\end{table*}

As shown in Table~\ref{Tab:experiment-2}, the practical D$^2$HPG exhibits robustness comparable to that of BPQL under moderate stochasticity. Although the practical D$^2$HPG exhibits noticeable performance degradation at $\sigma = 0.10$, this setting introduces substantial actuator-level perturbations, with a noise standard deviation equal to 10$\%$ of the maximum absolute action magnitude. Notably, BPQL also exhibits substantial performance degradation under the same perturbation level, indicating that such degradation is not specific to the practical D$^2$HPG. Overall, these results suggest that the delay-free MDP can remain a useful surrogate for the abstract MDP under moderate stochasticity, supporting the applicability of practical D$^2$HPG beyond the deterministic limit in which the abstraction is exact.

\subsection{Ablation Study}

We conduct ablation studies to analyze D$^2$HPG from three perspectives. First, we examine the benefit of incorporating an auxiliary regular-policy learner by comparing two variants of practical D$^2$HPG: D$^2$HPG-naive and D$^2$HPG-BPQL. Specifically, D$^2$HPG-naive optimizes the regular policy solely based on the homomorphic policy gradient, whereas D$^2$HPG-BPQL additionally employs BPQL as an auxiliary regular-policy learner. For brevity, D$^2$HPG-BPQL is referred to as D$^2$HPG throughout the main text. Second, we evaluate the sample overhead of principled D$^2$HPG relative to practical D$^2$HPG by measuring the number of time steps required to match the performance achieved by practical D$^2$HPG after one million time steps. This comparison is intended to quantify the sample overhead associated with explicit abstraction learning in principled D$^2$HPG and highlight the sample-efficiency benefit of bypassing this process in practical D$^2$HPG. Finally, we evaluate the computational overhead of D$^2$HPG-naive and D$^2$HPG-BPQL through wall-clock runtime comparisons with delayed RL baselines. Detailed results are provided in Appendix~\ref{ablation-study}.

\section{Limitations} \label{limitation}


\paragraph{General stochasticity}
Practical D$^2$HPG, a simplified variant of principled D$^2$HPG, is theoretically justified under deterministic dynamics. Although our experiments suggest that practical D$^2$HPG remains effective under moderate stochasticity, its effectiveness under stronger stochasticity remains to be established. Importantly, this limitation is specific to practical D$^2$HPG, whereas principled D$^2$HPG is grounded in the stochastic homomorphic policy gradient theorem \citep{HPG2} and supports policy optimization through an abstract MDP under general stochastic dynamics.

\paragraph{Random delays} The DHRL framework assumes a fixed delay, whereas real-world dynamic systems may experience randomly varying delays with unknown characteristics. Extending our framework to such settings remains an important direction for future work, potentially building on the conservative agent formulation of \citet{conservative-RL}, which reformulates bounded random-delay environments into their constant-delay surrogates, enabling conventional fixed-delay methods to be directly extended random-delay environments.

\section{Conclusion}

In this study, we revisited delayed reinforcement learning as an abstraction problem and proposed delayed homomorphic reinforcement learning (DHRL), an RL framework grounded in MDP homomorphisms that introduces a belief-equivalence relation to identify and collapse control-redundant augmented states, thereby alleviating the sample-complexity issue arising from state-space explosion in augmentation-based approaches. In finite domains, our theoretical analysis demonstrated that this framework enables exact abstraction under deterministic dynamics, preserving optimality while recovering the delay-free sample-complexity order. Under stochastic dynamics, the framework admits approximate abstraction with a value-loss bound on the resulting policy. Building on these theoretical findings, we developed deep delayed homomorphic policy gradient (D$^2$HPG) for continuous domains, including a practical instantiation that bypasses explicit abstraction learning by using the delay-free MDP as a surrogate for the abstract MDP. Experiments on continuous-control tasks in MuJoCo benchmark demonstrated that D$^2$HPG achieves strong overall performance compared with representative delayed RL baselines, particularly under long delays. Overall, our findings highlight the potential of belief-based MDP homomorphisms to exploit control-redundant structure in the augmented state space and thereby alleviate the sample-complexity burden induced by state-space explosion, a fundamental challenge for augmentation-based approaches.


\begin{ack}
This work was supported in part by the National Research Foundation of Korea (NRF) grant funded by the Korea government (MSIT) (No. RS-2026-25481115), in part by development of sensor and robot technology to support life search and fire suppression activities at firefighting sites program of National Fire Agency (Project No. 20026197, Development of track-based mobile robot for fire suppression and development of firefighting environment demonstration technology), in part by MSIT under the ITRC (Information Technology Research Center) support program (IITP-2025-RS-2024-00438430, 30$\%$), and in part by the National Research Foundation of Korea (NRF) grant funded by the Korean government (MSIT) in 2026 (Grant No. RS-2026-25561904; Project Title: Vision-Language-Model-based adaptive long-horizon task planning for autonomous multi-humanoid collaboration).
\end{ack}



{
\bibliographystyle{apalike}
\bibliography{reference}
}

\newpage 

\appendix

\section{Related Work} \label{related-work}

In real-world RL applications, feedback from the environment can often be delayed, which breaks the Markov assumption and degrades the performance of RL agents \citep{hwangbo, mahmood}. To cope with delays, delayed RL typically follows one of two strategies: augmentation-based approaches that incorporate sufficient information into the state to recover Markovian dynamics, and model-based approaches that infer unobserved states using an approximate dynamics model of the underlying delay-free MDP.

Augmentation-based approaches incorporate action histories into the state to construct an equivalent delay-free MDP, thereby enabling the direct application of conventional RL algorithms \citep{augmented2, augmented, DIDA}. However, the state augmentation enlarges the state space as the delay increases, resulting in an increased sample-complexity burden. To mitigate this issue, \citet{BPQL} propose an alternative value representation over augmented states that is evaluated with respect to the original state space, thereby improving sample efficiency. \citet{AD-SAC} uses auxiliary tasks with shorter delays to facilitate critic learning under longer delays. \citet{signal} proposes delay-reconciled critic training, which time-calibrates trajectories to recover non-delayed information for offline critic updates.

Despite these advances, the actors in these approaches still require augmented states because the current observation is unavailable at inference time. Consequently, these methods primarily alleviate the sample-complexity burden associated with critic learning, while the actor remains exposed to the enlarged augmented state space. The closest work to ours is \citet{VDPO}, which proposes an iterative method that first learns a delay-free policy and then distills it into the regular policy via behavior cloning. However, policy distillation through behavior cloning can be sensitive to distribution mismatch between the delay-free and regular policies \citep{DAgger, mismatch}. These observations motivate a principled abstraction mechanism that alleviates the sample-complexity burden induced by state augmentation for both the actor and critic.

Model-based approaches constitute another line of work on delay-compensation that seeks to recover information about unobserved states using an approximate delay-free dynamics model \citep{mbs, Duel-RNN, At-Human, Delay-aware, acting}. These methods learn a delay-free dynamics model from transition data and infer unobserved states through recursive one-step predictions. While these methods avoid operating over an enlarged augmented state space, their performance depends on the accuracy of the learned dynamics model. Moreover, prediction errors can accumulate over recursive rollouts, particularly in stochastic environments, potentially affecting policy performance and stability.

\section{Intuition Behind Principled D$^2$HPG}

To provide an intuition behind principled D$^2$HPG, we consider a stochastic one-dimensional navigation task with $s_{t+1} = s_t + a_t + \epsilon_t$ and $\epsilon_t \sim \mathcal{N}(0,\sigma^2)$, where the agent receives the reward $r_t=-|s_t|$. Under a one-step observation delay, the augmented state is $x_t=(s_{t-1},a_{t-1})$ and induces the belief
\begin{align}
    \nonumber
    b_\Delta(\cdot \mid x_t) = \mathcal{N}(s_{t-1}+a_{t-1},\sigma^2).
\end{align}
Then, let us consider two distinct augmented states $x_1=(-0.2,0.3)$ and $x_2=(0.0,0.1)$. Although these augmented states are different, both induce the same belief $\mathcal{N}(0.1,\sigma^2)$ over the current state and are therefore belief-equivalent. In other words, the two different observations and action histories represented by $x_1$ and $x_2$ provide the same information about the current state, making them redundant for control. The key idea of principled D$^2$HPG is to exploit this control-redundancy by mapping such belief-equivalent augmented states to the same or similar abstract representations.

Principled D$^2$HPG learns such a structure directly from transitions collected in the regular MDP. In this example, the homomorphism map learns to map belief-equivalent augmented states such as $x_1$ and $x_2$ to the same or similar abstract representations. The abstract reward and transition models then learn the reward and transition behavior associated with these abstract representations. In this way, the learned abstract MDP preserves the information relevant to control while removing redundant distinctions between augmented states. The learned abstraction components are then applied to regular transitions to construct abstract transitions, which are used to train the abstract actor and critic. Finally, the regular policy is aligned with the learned abstract policy.

\newpage 

\section{Proofs and Definitions} \label{Appendix-A}

We recall the notions of a reward-respecting partition and the stochastic substitution property (SSP) introduced in \citet{symmetry}, which are used in the proof of Proposition~3.3.

\begin{definition}
A partition $B$ of an MDP $\mathcal{M} = (\mathcal{S},\mathcal{A}, \Psi,\mathcal{P}, \mathcal{R})$ is said to be reward-respecting if $(s_1, a_1) \equiv_B (s_2, a_2)$ implies $\mathcal{R}(s_1, a_1) = \mathcal{R}(s_2, a_2)$ for all $(s_1, a_1), (s_2, a_2) \in \Psi$.
\end{definition}

\begin{definition}
A partition $B$ of an MDP $\mathcal{M} = (\mathcal{S},\mathcal{A}, \Psi,\mathcal{P}, \mathcal{R})$ satisfies the stochastic substitution property if for all $(s_1, a_1), (s_2, a_2) \in \Psi$, $(s_1, a_1) \equiv_B (s_2, a_2) $ implies $\mathcal{P}(G \mid s_1, a_1) = \mathcal{P}(G \mid s_2, a_2)$ for all $G \in B|S$. Here, we use the shorthand $\mathcal{P}(G \mid s, a) := \sum_{s' \in G}\mathcal{P}(s' \mid s, a)$.
\end{definition}

\begin{lemma}
    Let $B$ be a reward-respecting SSP partition of an MDP $\mathcal{M} = (\mathcal{S},\mathcal{A}, \Psi,\mathcal{P}, \mathcal{R})$. Then there exists a finite MDP homomorphism from $\mathcal{M}$ to the quotient MDP $\mathcal{M}/B$.
\end{lemma}

\subsection{Proof for Proposition 3.3} \label{Proof-Proposition-1}

\paragraph{Proposition 3.3.}
\textit{Let $B_\Delta$ be the partition of $\Psi_\Delta$ induced by the belief-equivalence relation on $\mathcal{X}$. If the dynamics are deterministic, then the partition $B_\Delta$ is reward-respecting and satisfies SSP.}

\begin{proof}[Proof]Let $\mathcal{M}^+=(\mathcal{S}, \mathcal{A}, \Psi, \mathcal{P}, \mathcal{R}, \Delta)$ be a delayed MDP and $\mathcal{M}_\Delta = (\mathcal{X}, \mathcal{A}, \Psi_\Delta, \mathcal{P}_\Delta, \mathcal{R}_\Delta)$ be its regular MDP. By definition~\ref{b-equivalent}, the two augmented states $x_1, x_2 \in \mathcal{X}$ are belief-equivalent if it satisfies
\begin{align}
    \nonumber
    b_\Delta(\cdot \mid x_1) = b_\Delta(\cdot \mid x_2),
\end{align} where the induced equivalence classes are referred to as belief classes. Since $x_1$ and $x_2$ in the same belief class induce the same belief over the underlying state space $\mathcal{S}$, we have 
\begin{align}
\nonumber
\mathcal{R}_\Delta(x_1, a) = \mathbb{E}_{b_\Delta(\cdot \mid x_1)}[\mathcal{R}(s, a)] = \mathbb{E}_{b_\Delta(\cdot \mid x_2)}[\mathcal{R}(s, a)] = \mathcal{R}_\Delta(x_2, a),
\end{align} for all $a \in \mathcal{A}_x$. Thus, the partition $B_\Delta$ is a reward-respecting partition for $\Psi_\Delta$. Subsequently, to verify that the partition $B_\Delta$ satisfies the SSP, we need to show that, for all $a \in \mathcal{A}_x$ and $G_x \in B_{\Delta}|{\mathcal{X}}$,
\begin{align}
\nonumber
\sum_{x' \in G_x}\mathcal{P}_\Delta(x' \mid x_1, a) = \sum_{x' \in G_x}\mathcal{P}_\Delta(x' \mid x_2, a),
\end{align}
whenever $x_1$ and $x_2$ belong to the same belief class. Let $\mathsf{P}(\mathcal{S})$ denote the probability simplex over the state space $\mathcal{S}$, and define the belief-update operator $\mathcal{F}: \mathsf{P}(\mathcal{S}) \times \mathcal{A} \to \mathsf{P}(\mathcal{S})$, which maps a belief $b \in \mathsf{P}(\mathcal{S})$ and an action $a \in \mathcal{A}$ to the next belief $b' \in \mathsf{P}(\mathcal{S})$ induced by the underlying deterministic transition kernel $\mathcal{P}$. Let $\tau_\Delta: \mathcal{X} \times \mathcal{A} \to \mathcal{X}$ be the augmented transition kernel induced by $\mathcal{P}$. Then, for any augmented state $x \in \mathcal{X}$, the next belief $b'$ is given by
\begin{align}
    \nonumber
    \mathcal{F}\big(b_\Delta(\cdot \mid x), a\big) = b'_\Delta(\cdot \mid \tau_\Delta(x, a)).
\end{align} Accordingly, the belief-equivalence $b_\Delta(\cdot \mid x_1) = b_\Delta(\cdot \mid x_2)$ implies
\begin{align}
    \nonumber
    \mathcal{F}\big(b_\Delta(\cdot \mid x_1), a\big) = \mathcal{F}\big(b_\Delta(\cdot \mid x_2), a\big) \; \Longrightarrow \; b'_\Delta(\cdot \mid \tau_\Delta(x_1, a)) = b'_\Delta(\cdot \mid \tau_\Delta(x_2, a)),
\end{align} yielding the belief-equivalence relation over the next augmented states, i.e., $\tau_\Delta(x_1, a) \equiv_{b_\Delta} \tau_\Delta(x_2, a)$. From this result, we have the probabilities of transitioning from the belief-equivalent augmented states $x_1, x_2$ to the block $G_x$ that satisfy
\begin{align}
\nonumber
\sum_{x' \in G_x} \mathcal{P}_\Delta(x' \mid x_1, a) = \mathds{1}\{\tau_\Delta(x_1,a) \in G_x\} = \mathds{1}\{\tau_\Delta(x_2,a) \in G_x\} = \sum_{x' \in G_x} \mathcal{P}_\Delta(x' \mid x_2, a),
\end{align} for all $a \in \mathcal{A}_x$ and $G_x \in B_\Delta\vert \mathcal{X}$. This suggests that starting from any two augmented states in the same belief class, the probability of transitioning to every other belief class under the same action is identical. Consequently, the partition $B_\Delta$ of $\Psi_\Delta$ is a reward-respecting and satisfies SSP. This completes the proof.
\end{proof}

\subsection{Proof of Proposition 3.5}\label{Proof-Proposition-2}

\paragraph{Proposition 3.5.}

\textit{The compression ratio of the abstract state space $\bar{\mathcal{X}}$ induced by belief-equivalence on the augmented state space $\mathcal{X}$ satisfies, under deterministic dynamics,}
\begin{align}
    \nonumber
    \zeta := \frac{|\bar{\mathcal{X}}|}{|\mathcal{X}|} 
    \leq \frac{1}{|\mathcal{A}|^{\Delta}},
\end{align}
for any $\Delta \in \mathbb{N}$. In particular, this implies that $|\bar{\mathcal{X}}| \leq |\mathcal{S}|$.

\begin{proof} 

Let $\mathcal{M}^+=(\mathcal{S}, \mathcal{A}, \Psi, \mathcal{P}, \mathcal{R}, \Delta)$ be a delayed MDP, $\mathcal{M}_\Delta = (\mathcal{X}, \mathcal{A}, \Psi_\Delta, \mathcal{P}_\Delta, \mathcal{R}_\Delta)$ be the regular reformulation of $\mathcal{M}^+$, and $\bar{\mathcal{M}}_\Delta = (\bar{\mathcal{X}}, \bar{\mathcal{A}}, \bar{\Psi}_\Delta, \bar{\mathcal{P}}_\Delta, \bar{\mathcal{R}}_\Delta)$ be the homomorphic image of $\mathcal{M}_\Delta$ induced by the belief-equivalence relation. Under deterministic dynamics, the belief $b_\Delta(\cdot\mid x)$ at an augmented state $x \in \mathcal{X}$ collapses to a Dirac measure at the underlying state $s \in \mathcal{S}$ corresponding to $x$. Hence, there exists a map $\mathcal{F}_\Delta:\mathcal{X}\to\mathcal{S}$ such that $b_\Delta(\cdot\mid x) = \delta_{\mathcal{F}_\Delta(x)}$ for all $x \in \mathcal{X}$. Since the abstract state space $\bar{\mathcal{X}}$ can be identified by the set of underlying states represented by the augmented states, its cardinality satisfies
\begin{align}
\nonumber
\vert \bar{\mathcal{X}} \vert = \vert \mathrm{im}(\mathcal{F}_{\Delta})\vert = \vert \{ s\in\mathcal{S} \mid \exists x\in\mathcal{X}\;\text{s.t.}\; \mathcal{F}_{\Delta}(x)=s \} \vert \leq \vert \mathcal{S} \vert,
\end{align} where $\mathrm{im}(\mathcal{F}_{\Delta})$ denotes the image of $\mathcal{F}_\Delta$. The compression ratio $\zeta$ is thus given by
\begin{align}
\nonumber
\zeta \coloneqq \frac{|\bar{\mathcal{X}}|}{|\mathcal{X}|} = \frac{|\mathrm{im}(\mathcal{F}_{\Delta})|}{|\mathcal{S}||\mathcal{A}|^{\Delta}}
\leq \frac{|\mathcal{S}|}{|\mathcal{S}||\mathcal{A}|^{\Delta}} = \frac{1}{|\mathcal{A}|^{\Delta}},
\end{align} where $\vert \mathcal{X}\vert = \vert \mathcal{S} \vert \vert \mathcal{A} \vert^{\Delta}$. This completes the proof.   
\end{proof}

\subsection{Proof for Corollary 3.6} \label{Proof-Proposition-3}

\paragraph{Corollary 3.6.} \textit{Under deterministic dynamics, the sample complexity of $Q$-learning on the abstract MDP $\bar{\mathcal{M}}_\Delta$ is upper-bounded by 
\begin{align}
        \nonumber
        O\Big(\frac{\log (\vert \mathcal{S}\vert\vert\mathcal{A} \vert)}{\varepsilon^{2.5}(1 - \gamma)^5}\Big),
\end{align} which is of the same order as the sample-complexity bound for the delay-free MDP.}

\begin{proof} Let $\mathcal{M}^+=(\mathcal{S}, \mathcal{A}, \Psi, \mathcal{P}, \mathcal{R}, \Delta)$ be a delayed MDP, $\mathcal{M}_\Delta = (\mathcal{X}, \mathcal{A}, \Psi_\Delta, \mathcal{P}_\Delta, \mathcal{R}_\Delta)$ be the regular reformulation of $\mathcal{M}^+$, and $\bar{\mathcal{M}}_\Delta = (\bar{\mathcal{X}}, \bar{\mathcal{A}}, \bar{\Psi}_\Delta, \bar{\mathcal{P}}_\Delta, \bar{\mathcal{R}}_\Delta)$ be the homomorphic image of $\mathcal{M}_\Delta$ induced by the belief-equivalence relation.  As formalized in \citet{Speed-QL}, the $Q$-learning sample complexity bound for the abstract MDP $\bar{\mathcal{M}}_\Delta$ is given by  
\begin{align}
    O\Big(\frac{\log (\vert \bar{\mathcal{X}}\vert\vert\bar{\mathcal{A}} \vert)}{\varepsilon^{2.5}(1 - \gamma)^5}\Big). \label{bound}
\end{align} Under deterministic dynamics, it follows that 
\begin{align}
\nonumber
\log(\vert \bar{\mathcal{X}} \vert  \vert \bar{\mathcal{A}} \vert ) &\leq \log(\vert \mathcal{X} \vert  \vert \bar{\mathcal{A}} \vert / \vert \mathcal{A} \vert^\Delta) \quad (\because \vert \bar{\mathcal{X}} \vert \leq \vert \mathcal{X} \vert /\vert \mathcal{A} \vert^{\Delta})
\\ \nonumber&= \log(\vert \mathcal{S} \vert \vert \mathcal{A} \vert^\Delta \vert \bar{\mathcal{A}} \vert / \vert \mathcal{A} \vert^\Delta) 
\\ \nonumber&= \log(\vert \mathcal{S} \vert \vert \bar{\mathcal{A}} \vert).
\end{align}Since $\bar{\mathcal{A}}$ is the image of $\mathcal{A}$ under the surjective action mapping, we have $\vert\bar{\mathcal{A}}\vert \leq \vert \mathcal{A}\vert$. Therefore,
\begin{align}
    \nonumber
    \log(\vert \mathcal{S} \vert \vert \bar{\mathcal{A}} \vert) \leq \log(\vert \mathcal{S} \vert \vert \mathcal{A} \vert).
\end{align}
Substituting this into the sample complexity bound in Eq.~\eqref{bound} gives
\begin{align}
\nonumber
O\Big(\frac{\log\big(|\mathcal{S}||\mathcal{A}|\big)}{\varepsilon^{2.5}(1-\gamma)^5}\Big),
\end{align} where the $\Delta$-dependent factor inside the logarithm is eliminated. This concludes the proof.
\end{proof}

\subsection{Proof for Proposition 4.2} \label{proposition_4_2}

\paragraph{Proposition 4.2.} \textit{Under deterministic dynamics, the delay-free MDP coincides with the homomorphic image of the regular MDP. Consequently, the abstract policy can be identified with the delay-free policy, eliminating the need to learn the homomorphism map and abstract dynamics model.}

\begin{proof}
Let $\tau: \mathcal{S} \times \mathcal{A} \to \mathcal{S}$ and $\tau_\Delta: \mathcal{X} \times \mathcal{A} \to \mathcal{X}$ denote the deterministic transition functions of $\mathcal{M}$ and $\mathcal{M}_\Delta$, respectively. Under deterministic dynamics, the belief induced by any augmented state $x \in \mathcal{X}$ collapses to a Dirac measure on $\mathcal{S}$, i.e., there exists a map $\mathcal{F}_\Delta: \mathcal{X} \to \mathcal{S}$ such that
\begin{align}   
    \nonumber
    b_\Delta(\cdot \mid x) = \delta_{\mathcal{F}_\Delta(x)}.
\end{align} Define the state map $f_x: \mathcal{X} \to \mathcal{S}$ by $f_x(x)=\mathcal{F}_\Delta(x)$ and let the action map $g_x:\mathcal{A}_x \to \mathcal{A}$ be the identity map, i.e., $g_x(a)=a$. Then we have
\begin{align}
    \nonumber
    \mathcal{R}_\Delta(x,a) = \mathbb{E}_{b_\Delta(\cdot \mid x)}[\mathcal{R}(s,a)] = \mathcal{R}(f_x(x),g_x(a)).
\end{align}
Moreover,
\begin{align}
    \nonumber
    f_x(\tau_\Delta(x,a))=\tau(f_x(x),g_x(a)),
\end{align}
for all $(x,a)\in\Psi_\Delta$. Since both $\tau$ and $\tau_\Delta$ are deterministic, this identity implies the block-transition condition in Definition~\ref{def3.1}. Therefore, the pair $(f_x,g_x)$ defines an exact MDP homomorphism from the reachable part of $\mathcal{M}_\Delta$ to $\mathcal{M}$. This justifies the use of the delay-free MDP as a useful surrogate for the abstract MDP. This completes the proof.
\end{proof}

\subsection{Approximate Abstract MDP} \label{approximate-abstract}

Under deterministic dynamics, a belief-equivalence relation on the augmented state space induces an exact reward-respecting SSP partition, which yields an exact homomorphic image of the regular MDP. However, the exact equality $b_\Delta(\cdot \mid x)=b_\Delta(\cdot \mid x')$ rarely holds under stochastic dynamics. We therefore consider an approximate MDP homomorphism formalized in \citet{approximate-MDP}.

Let $\mathcal{B}_\varepsilon$ be the $\varepsilon$-abstract partition of the augmented state space $\mathcal{X}$ such that any two augmented states $x,x'\in\mathcal{X}$ in the same block satisfy
\begin{align}
\nonumber
\|b_\Delta(\cdot \mid x)-b_\Delta(\cdot \mid x')\|_{\mathrm{TV}} \le \varepsilon,
\end{align}
for $\varepsilon\in[0,1)$. Let $\bar{\mathcal{X}}_\varepsilon := \mathcal{X}/\mathcal{B}_\varepsilon$ be the corresponding abstract state space, and let
\begin{align}
\nonumber
f_x:\mathcal{X}\to\bar{\mathcal{X}}_\varepsilon
\end{align}
be the induced surjective state map. We then identify the abstract action space with the original one by setting $\bar{\mathcal A}:=\mathcal A$, and define the surjective action map $g_x:\mathcal{A}_x \to \bar{\mathcal{A}}$ as the identity, i.e., $g_x(a)=a$ for all $x \in \mathcal{X}$ and $a \in \mathcal{A}_x$. Building on these surjective maps, we can construct an approximate abstract MDP by aggregating state-action pairs that share the same abstract image. Let us define
\begin{align}
    \nonumber
    \Omega(\bar{x},\bar{a}) := \{(x,u)\in \Psi_\Delta : f_x(x)=\bar{x},\ g_x(u)=\bar{a}\},
\end{align} for any $(\bar{x},\bar{a}) \in \bar{\mathcal{X}}_\varepsilon \times \bar{\mathcal{A}}$ and let $w_{\bar{x},\bar{a}}(\cdot,\cdot)$ be a probability distribution over $\Omega(\bar{x},\bar{a})$. We then define the abstract reward and transition by weighted averages such that 
\begin{align}
    \nonumber
    \bar{\mathcal{R}}_{\varepsilon}(\bar{x},\bar{a})
    &:=
    \sum_{(x,u)\in \Omega(\bar{x},\bar{a})}
    w_{\bar{x},\bar{a}}(x,u)\,
    \mathcal{R}_\Delta(x,u),  \\ \nonumber
    \bar{\mathcal{P}}_{\varepsilon}(\bar{x}' \mid \bar{x},\bar{a})
    &:=
    \sum_{(x,u)\in \Omega(\bar{x},\bar{a})}
    w_{\bar{x},\bar{a}}(x,u)
    \sum_{x'' \in \mathcal{X}: f_x(x'')=\bar{x}'}
    \mathcal{P}_\Delta(x'' \mid x,u).
\end{align}
This yields the approximate abstract MDP
\begin{align}
    \nonumber
    \bar{\mathcal{M}}_{\varepsilon} = (\bar{\mathcal{X}}_\varepsilon,\bar{\mathcal{A}},\bar{\Psi}_\varepsilon,\bar{\mathcal{P}}_{\varepsilon},\bar{\mathcal{R}}_{\varepsilon}),
\end{align}
where $\bar{\Psi}_\varepsilon := \{(\bar{x},\bar{a})\in \bar{\mathcal{X}}_\varepsilon \times \bar{\mathcal{A}} : \Omega(\bar{x},\bar{a})\neq \emptyset\}, \bar{\mathcal{P}}_{\varepsilon}:\bar{\Psi}_\varepsilon \times \bar{\mathcal{X}}_\varepsilon \rightarrow [0, 1]$, and $\bar{\mathcal{R}}_{\varepsilon}:\bar{\Psi}_\varepsilon \rightarrow \mathbb{R}$.

\newpage

\section{Experimental Details} \label{Appendix-detail}

\subsection{Hyperparameters}

The hyperparameters of D$^2$HPG algorithm are aligned with those used in \citet{BPQL, HPG2}. Because all baselines included in our experiments are built upon SAC~\citep{sac}, we adopt a shared hyperparameter configuration recommended by BPQL and VDPO across all algorithms to ensure a fair comparison. The detailed configuration is reported in Table~\ref{hyperparameter}.

\begin{table}[!h]
\caption{Hyperparameters for the baseline algorithms.} \label{hyperparameter}
\centering
\renewcommand{\arraystretch}{1.3} 
\begin{tabularx}{\textwidth}{X c} 
\hline
\textbf{Hyperparameters} & Values \\
\hline
Actor network  & $256, 256$ \\
Critic network & $256, 256$ \\
Learning rate (actor)  & $3 \times 10^{-4}$ \\
Learning rate (critic) & $3 \times 10^{-4}$ \\
Temperature ($\alpha$)   & $0.2$ \\
Target smoothing coefficient ($\beta$) & $0.005$ \\
Discount factor ($\gamma$)   & $0.99$ \\
Replay buffer size   & $10^6$ \\
Batch size      & $256$ \\
Target entropy       & $-\text{dim}(\mathcal{A})$ \\
Optimizer            & Adam \citep{adam} \\
Total time steps     & $10^6$ \\
\hline
\end{tabularx}
\end{table}

\subsection{Environment Details} \label{Benchmark-detail}

\begin{table*}[!h] 
\caption{Environment details of the MuJoCo benchmark.} \label{environmental-1}
\centering
\renewcommand{\arraystretch}{1.3} 
\begin{tabularx}{\textwidth}{X cccc} 
\hline
\textbf{Environment} & State dimension & Action dimension & Action range & Time step (s) \\
\hline
Ant-v3              & $27$              & $8$           & $[-1.0, 1.0]$ & $0.05$          \\ 
HalfCheetah-v3      & $17$              & $6$           & $[-1.0, 1.0]$ & $0.05$          \\
Walker2d-v3         & $17$              & $6$           & $[-1.0, 1.0]$ & $0.008$         \\
Hopper-v3           & $11$              & $3$           & $[-1.0, 1.0]$ & $0.008$         \\
Humanoid-v3         & $376$             & $17$          & $[-0.4, 0.4]$ & $0.015$         \\
InvertedPendulum-v2 & $4$               & $1$           & $[-3.0, 3.0]$ & $0.04$          \\
\hline
\end{tabularx}
\end{table*}

\begin{figure}[!h]
    \centering
    \begin{subfigure}[b]{0.16\textwidth}
        \centering
        \includegraphics[width=\textwidth]{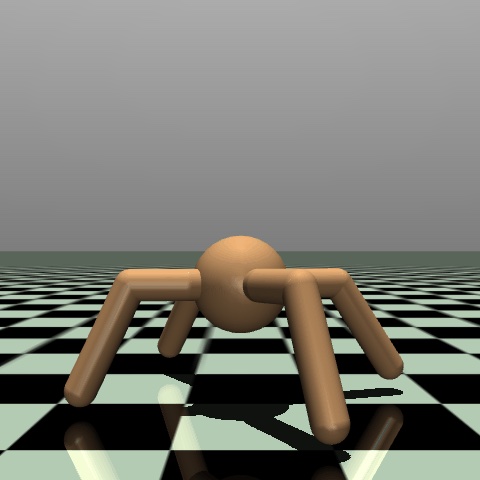}
        \caption{}
    \end{subfigure}
    \hfill 
    \begin{subfigure}[b]{0.16\textwidth}
        \centering
        \includegraphics[width=\textwidth]{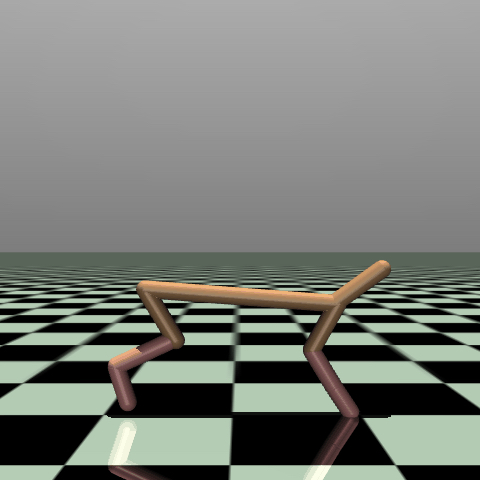}
        \caption{}
    \end{subfigure}
    \hfill
    \begin{subfigure}[b]{0.16\textwidth}
        \centering
        \includegraphics[width=\textwidth]{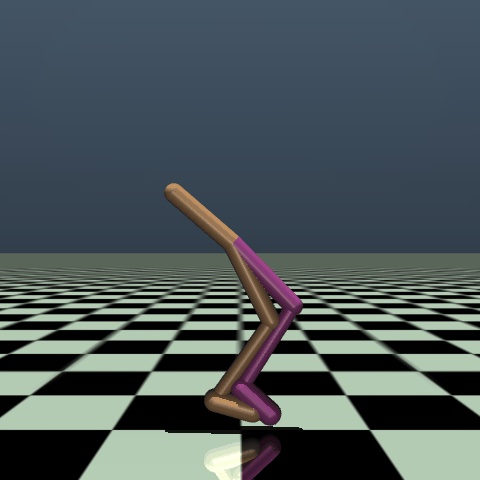}
        \caption{}
    \end{subfigure}
    \hfill
    \begin{subfigure}[b]{0.16\textwidth}
        \centering
        \includegraphics[width=\textwidth]{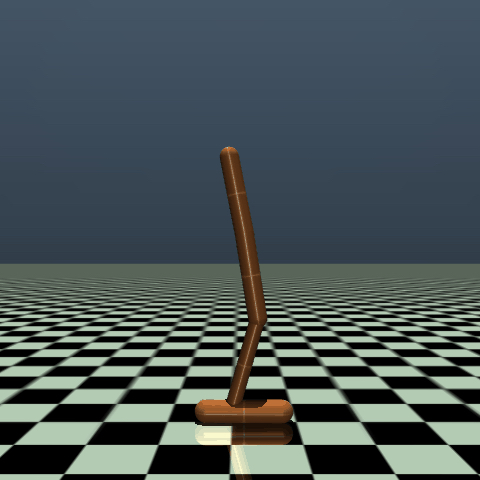}
        \caption{}
    \end{subfigure}
    \hfill
    \begin{subfigure}[b]{0.16\textwidth}
        \centering
        \includegraphics[width=\textwidth]{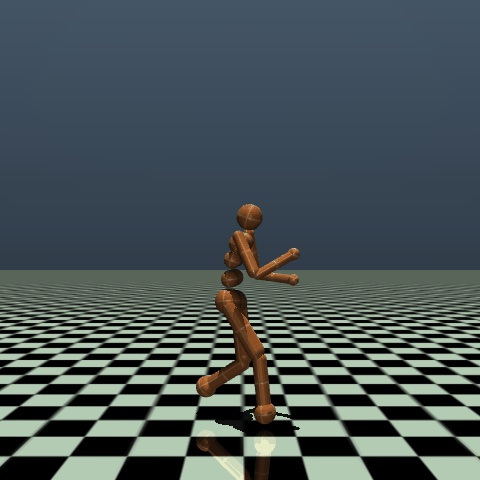}
        \caption{}
    \end{subfigure}
    \hfill
    \begin{subfigure}[b]{0.16\textwidth}
        \centering
        \includegraphics[width=\textwidth]{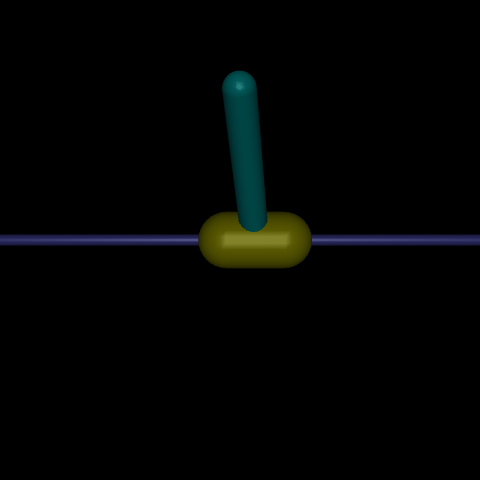}
        \caption{}
    \end{subfigure}

    \caption{Visual illustration of continuous control tasks in MuJoCo: (a) Ant-v3 (b) HalfCheetah-v3, (c) Walker2d-v3, (d) Hopper-v3, (e) Humanoid-v3, and (f) InvertedPendulum-v2.} 
\end{figure}

\newpage 

\subsection{Pseudo-code for D$^2$HPG} \label{pseudo-code}

Algorithm~\ref{algorithm:D2HPG} summarizes the proposed D$^2$HPG algorithm. We introduce a temporary buffer $\mathcal{B}$ to store recent observations, rewards, and actions, enabling the agent to construct time-aligned transition samples as delayed feedback becomes available. Note that the augmented reward $\tilde{r}_t$ can be estimated empirically from transition samples stored in the replay buffer $\mathcal{D}$, i.e., $\mathcal{R}_{\Delta}(x_t,a_t) \approx \mathbb{E}_{\mathcal{D}}\left[\mathcal{R}(s_t,a_t)\right]$.

In practical D$^2$HPG, the delay-free actor and critic are trained from retrospectively time-aligned transitions $(s_{t-\Delta}, a_{t-\Delta}, r_{t-\Delta}, s_{t-\Delta+1})$ in a model-free manner, without interacting with the underlying delay-free MDP. In this instantiation, line~24 is omitted, while line~26 corresponds to an optional auxiliary update of the regular policy. Code is available at \url{https://github.com/JongsooLee/DHRL}.


\begin{algorithm}[h] 
\caption{Deep delayed homomorphic policy gradient (D$^2$HPG)} \label{algorithm:D2HPG}
\begin{algorithmic}[1]
\STATE \textbf{Input:} policies $\pi^{\uparrow}_{\theta}, \bar{\pi}_{\theta'}$, critics $Q_{\phi_1}, Q_{\phi_2}$, target critics $Q_{\phi'_1}, Q_{\phi'_2}$, replay buffer $\mathcal{D}$, temporary buffer $\mathcal{B}$, delay $\Delta$, target smoothing coefficient $\beta$, replay warm-up size $N$, episodic length $H$, and total number of episodes $E$.
\STATE Initialize $\mathcal{B} \leftarrow \emptyset$, $\mathcal{D} \leftarrow \emptyset$
\FOR{episode $e=1$ to $E$}
\FOR{time step $t=1$ to $H$}
\IF{t $\leq \Delta$}
    \STATE select random action $a_t$  
    \STATE execute $a_t$ in the environment
    \STATE store $a_t$, the observed state, and reward in $\mathcal{B}$
\ELSE
    \STATE retrieve $s_{t-\Delta}, a_{t-\Delta}, a_{t-\Delta+1}, \dots, a_{t-1}$ from $\mathcal{B}$
    \STATE $x_{t} \leftarrow (s_{t-\Delta}, a_{t-\Delta}, a_{t-\Delta+1}, \dots, a_{t-1})$
    \STATE sample action $a_t \sim \pi^{\uparrow}_\theta(\cdot \mid x_{t})$
    \STATE execute $a_t$ in the environment
    \STATE store $a_t$, the observed state, and reward in $\mathcal{B}$
    \IF{t $> 2\Delta$}
    \STATE retrieve $s_{t-\Delta}, s_{t-\Delta+1}, s_{t-2\Delta}, s_{t-2\Delta+1}, a_{t-2\Delta}, a_{t-2\Delta+1}, \dots, a_{t-\Delta}, r_{t-\Delta}$ from $\mathcal{B}$ 
    \STATE $x_{t-\Delta} \leftarrow (s_{t-2\Delta}, a_{t-2\Delta}, a_{t-2\Delta+1}, \dots, a_{t-\Delta-1})$
    \STATE $x_{t-\Delta+1} \leftarrow (s_{t-2\Delta+1}, a_{t-2\Delta+1}, a_{t-2\Delta+2}, \dots, a_{t-\Delta})$
    \STATE $\mathcal{D} \leftarrow \mathcal{D} \cup \{(x_{t-\Delta}, s_{t-\Delta}, a_{t-\Delta}, r_{t-\Delta}, x_{t-\Delta+1}, s_{t-\Delta+1})\}$
    \ENDIF
\ENDIF
    \IF{$|\mathcal{D}| \ge N$}
    \STATE sample and permute a mini-batch $\mathcal{D}_i \sim \mathcal{D}$
    \STATE train the homomorphism map $h_{\xi, \omega}$ and abstract dynamics models $\bar{\mathcal{P}}_\tau, \bar{\mathcal{R}}_\nu$ via $\mathcal{L}_\text{lax} + \mathcal{L}_\text{h}$ 
    \STATE train the critics $Q_{\phi_1}$ and $Q_{\phi_2}$ via $\mathcal{L}_\text{regular-critic} + \mathcal{L}_\text{abstract-critic}$
    \STATE train the regular policy $\pi_\theta^{\uparrow}$ via a stochastic actor-critic algorithm {\color{gray}(optional)}
    \STATE train the abstract policy $\bar{\pi}_{\theta'}$ via the homomorphic policy gradient
    \STATE align the policies $\pi_\theta^{\uparrow}$ and $\bar{\pi}_{\theta'}$ via $\mathcal{L}_\text{lift}$
    \STATE update the target critics $Q_{\phi'_1}$ and $Q_{\phi'_2}$: $\phi'_1 \leftarrow \beta\phi_1 + (1-\beta)\phi'_1$, $\phi'_2 \leftarrow \beta\phi_2 + (1-\beta)\phi'_2$
    \ENDIF
\ENDFOR
\ENDFOR
\STATE \textbf{Output:}  policies $\pi^{\uparrow}_{\theta}, \bar{\pi}_{\theta'}$
\end{algorithmic}
\end{algorithm}

\newpage 
\section{Additional Results} \label{full-results}

\subsection{Performance Curves}

We report the performance curves of D$^2$HPG and the baseline algorithms on MuJoCo tasks under fixed delays $\Delta \in \{5,10,20\}$. VDPO achieves competitive performance on some tasks, but its performance is highly task-dependent and degrades as the delay increases. BPQL performs strongly under short to moderate delays but degrades substantially at $\Delta=20$. In contrast, D$^2$HPG exhibits more consistent performance across tasks and maintains strong performance even under long delays.

\begin{figure}[!h]
    \centering
    \includegraphics[width=1.0\textwidth]{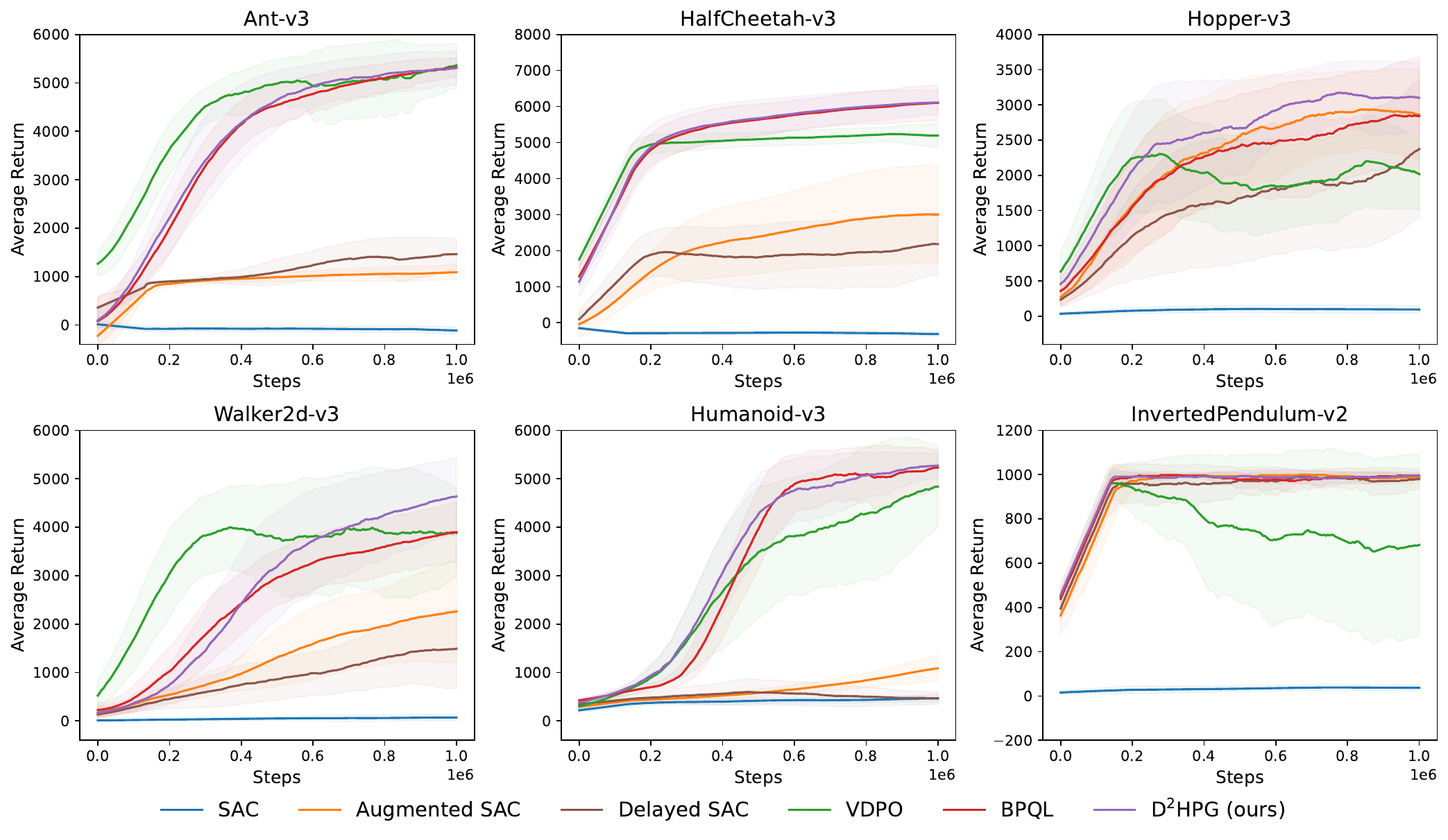} 
\caption{Performance curves of each algorithm on the MuJoCo benchmarks with $\Delta = 5$.}
\end{figure}

\begin{figure}[!h]
    \centering
    \includegraphics[width=1.0\textwidth]{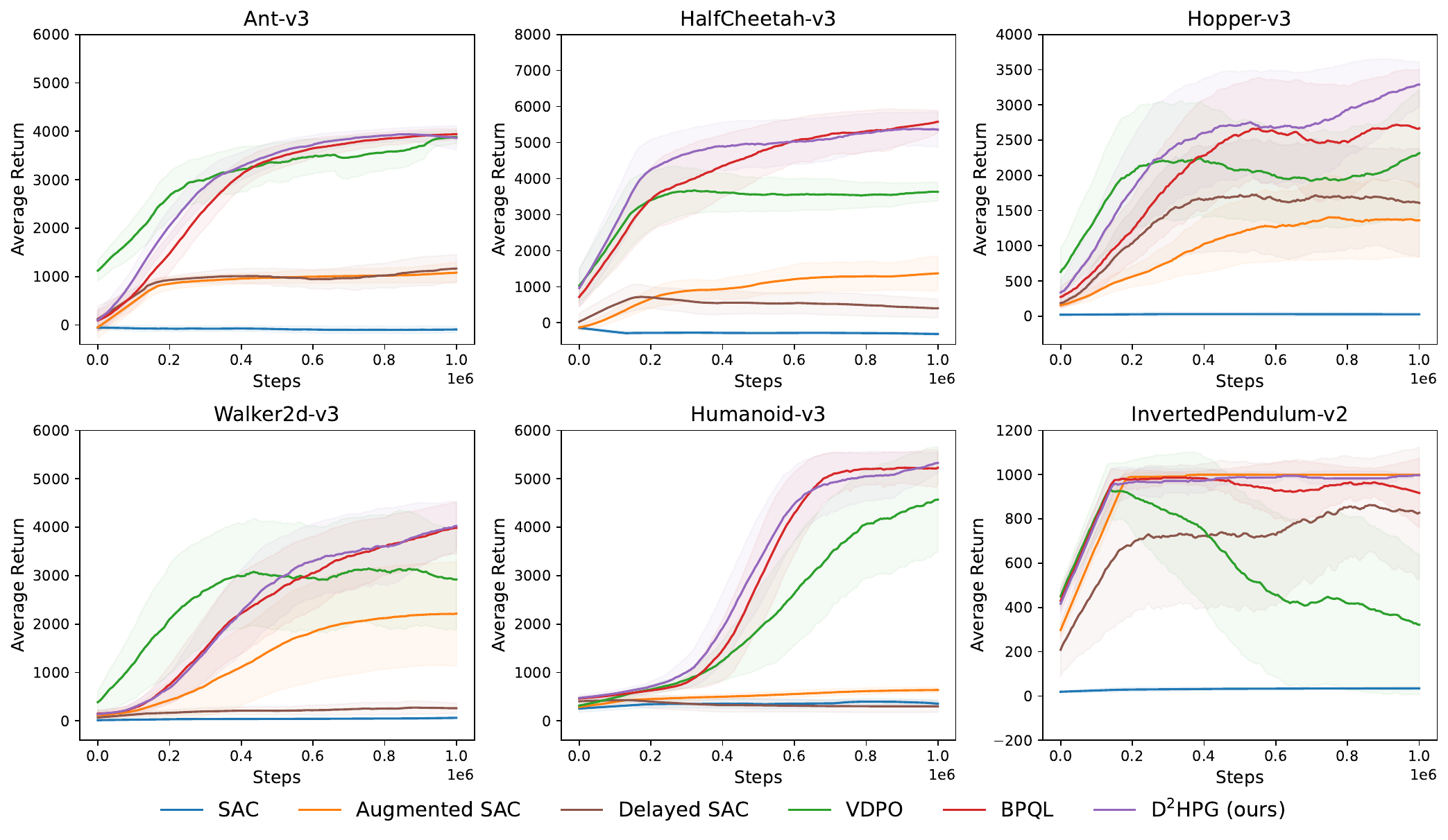}
\caption{Performance curves of each algorithm on the MuJoCo benchmarks with $\Delta = 10$.}
\end{figure}

\newpage 

\begin{figure}[!h]
    \centering
    \includegraphics[width=1.0\textwidth]{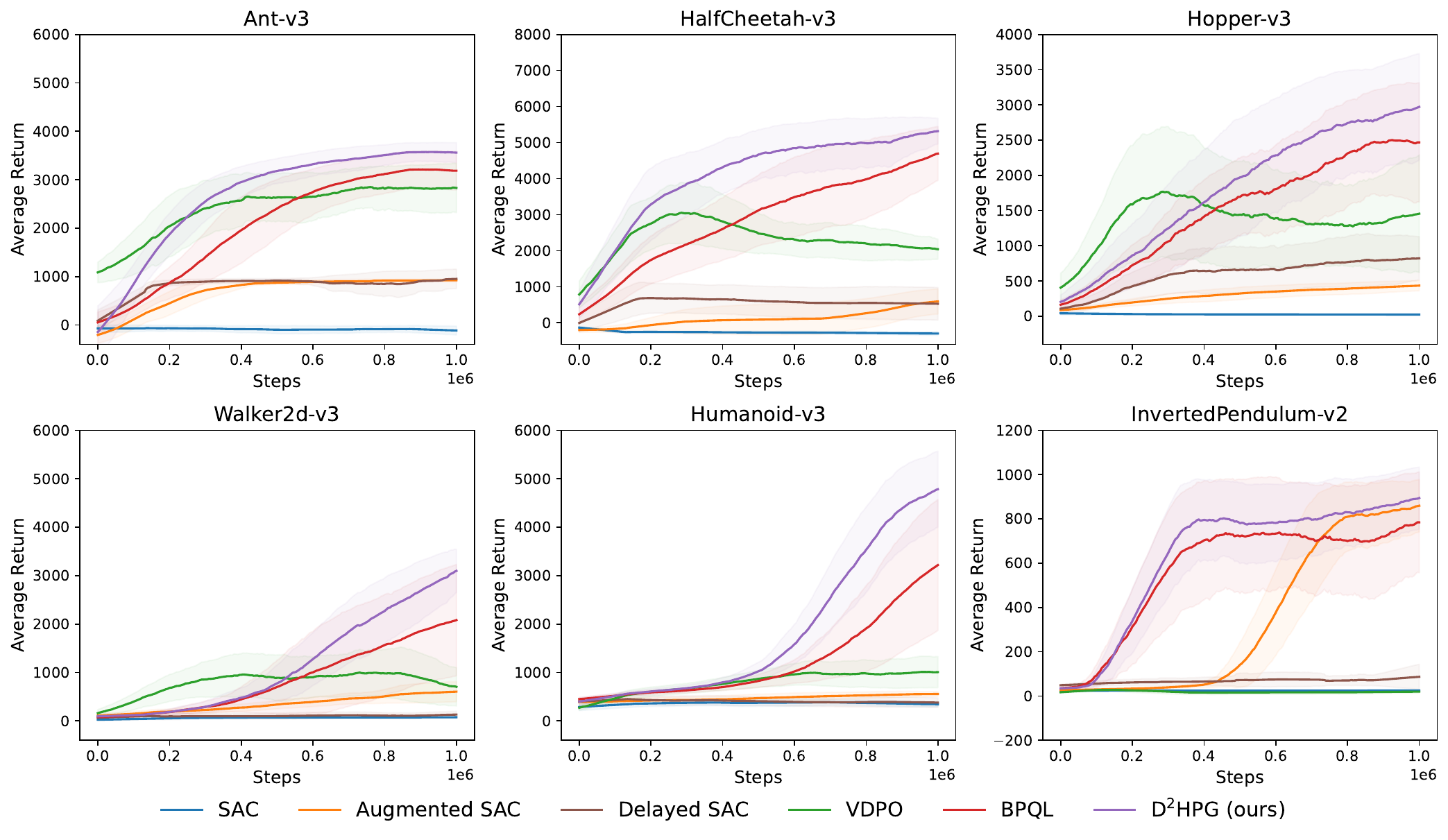} 
\caption{Performance curves of each algorithm on the MuJoCo benchmarks with $\Delta = 20$. }
\end{figure}

\subsection{Comparison with a Recent Model-Based Baseline}

We compare D$^2$HPG-BPQL with DFBT~\citep{DFBT}, a recent model-based delayed RL baseline. We refer to D$^2$HPG-BPQL as D$^2$HPG hereafter for brevity. Following the experimental setting of DFBT, we evaluate D$^2$HPG on three locomotion tasks—HalfCheetah-v2, Hopper-v2, and Walker2d-v2—with fixed delays $\Delta \in \{8, 32, 128\}$ over one million time steps and five random seeds. For comparison, we use the results reported in the original DFBT paper. 

In particular, we report the normalized performance defined in DFBT as $(R_{\mathrm{alg}}-R_{\mathrm{rnd}})/(R_{\mathrm{sac}}-R_{\mathrm{rnd}})$, where $R_{\mathrm{alg}}$, $R_{\mathrm{sac}}$, and $R_{\mathrm{rnd}}$ denote the returns of the evaluated algorithm, delay-free SAC, and random policy after one million time steps, respectively. Since the original DFBT paper does not report the delay-free SAC reference returns, we obtain $R_{\mathrm{sac}}$ using the Stable-Baselines3 SAC implementation \citep{stable-baseline}. As presented in Table~\ref{tab:dfbt-comparison}, D$^2$HPG remains competitive with DFBT at delays $\Delta \in \{8, 32\}$, while DFBT outperforms D$^2$HPG at the longest delay ($\Delta = 128$).

\begin{table}[!h]
\caption{Performance comparison between D$^2$HPG and DFBT. The best performance is highlighted.}
\label{tab:dfbt-comparison}
\small
\centering
\renewcommand{\arraystretch}{1.5}
\begin{tabular*}{\textwidth}{@{\extracolsep{\fill}}llccc@{}}
\hline
\textbf{Environment} & $\Delta$ & D$^2$HPG
& DFBT-SAC(1) & DFBT-SAC \\
\hline

Hopper-v2
& $8$   & ${\color{blue}1.02_{\pm 0.02}}$ & $0.92_{\pm 0.28}$ & $0.77_{\pm 0.18}$ \\
& $32$  & ${\color{blue}0.73_{\pm 0.17}}$ & $0.60_{\pm 0.23}$ & $0.68_{\pm 0.20}$ \\
& $128$ & $0.10_{\pm 0.02}$ & $0.16_{\pm 0.02}$ & ${\color{blue}0.20_{\pm 0.03}}$ \\
\hline

HalfCheetah-v2
& $8$   & ${\color{blue}0.47_{\pm 0.03}}$ & $0.38_{\pm 0.03}$ & $0.35_{\pm 0.12}$ \\
&$32$  & ${\color{blue}0.42_{\pm 0.06}}$ & $0.40_{\pm 0.07}$ & ${\color{blue}0.42_{\pm 0.03}}$ \\
& $128$ & $0.19_{\pm 0.05}$ & $0.40_{\pm 0.06}$ & ${\color{blue}0.41_{\pm 0.03}}$ \\
\hline

Walker2d-v2
& $8$   & ${\color{blue}0.99_{\pm 0.02}}$ & $0.95_{\pm 0.14}$ & ${\color{blue}0.99_{\pm 0.03}}$ \\
& $32$  & $0.21_{\pm 0.15}$ & $0.57_{\pm 0.21}$ & ${\color{blue}0.64_{\pm 0.10}}$ \\
& $128$ & $0.08_{\pm 0.04}$ & $0.38_{\pm 0.11}$ & ${\color{blue}0.40_{\pm 0.08}}$ \\
\hline
\end{tabular*}
\end{table}

\section{Ablation Study} \label{ablation-study}

\subsection{Auxiliary Regular-policy Learner} \label{d2hpg-variants}


Here, we compare the D$^2$HPG variants---D$^2$HPG-naive and D$^2$HPG-BPQL---to assess the benefit of incorporating an auxiliary regular-policy learner. As shown in Fig.~\ref{fig:variants}, D$^2$HPG-naive consistently outperforms augmented SAC, demonstrating the benefit of learning from the abstract MDP. Notably, incorporating BPQL as a separate regular-policy learner further improves performance, with D$^2$HPG-BPQL outperforming both D$^2$HPG-naive and BPQL across all tasks. Moreover, the performance advantage of D$^2$HPG-BPQL over BPQL becomes more pronounced as the delay $\Delta$ increases.

\begin{figure*}[h]
    \centering
    \begin{subfigure}[t]{1.0\textwidth}
        \centering
        \includegraphics[width=\textwidth]{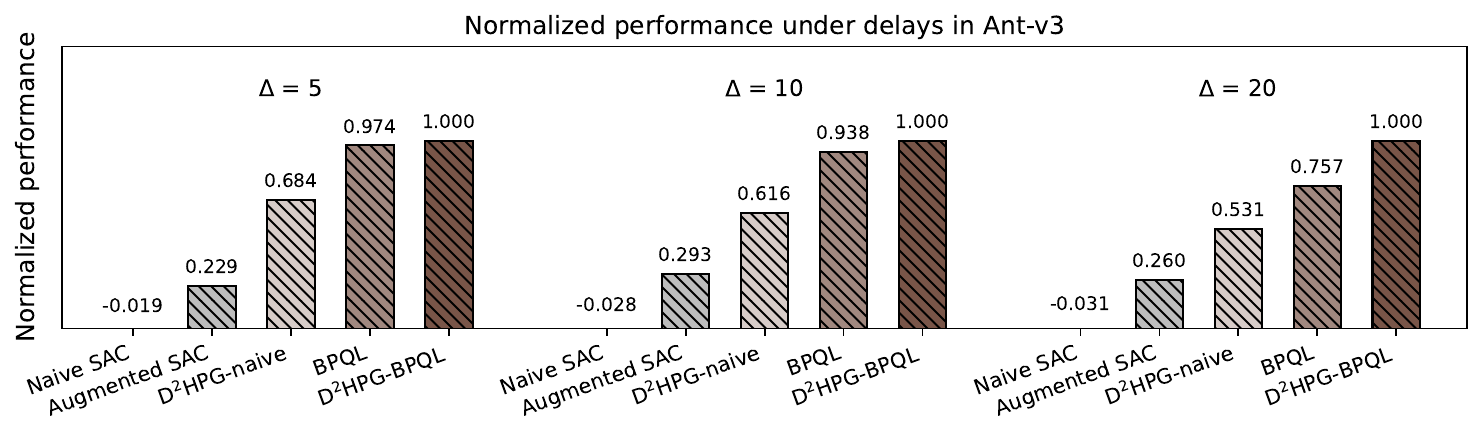}
        \caption{Ant-v3}
    \end{subfigure}
    \begin{subfigure}[t]{1.0\textwidth}
        \centering
        \includegraphics[width=\textwidth]{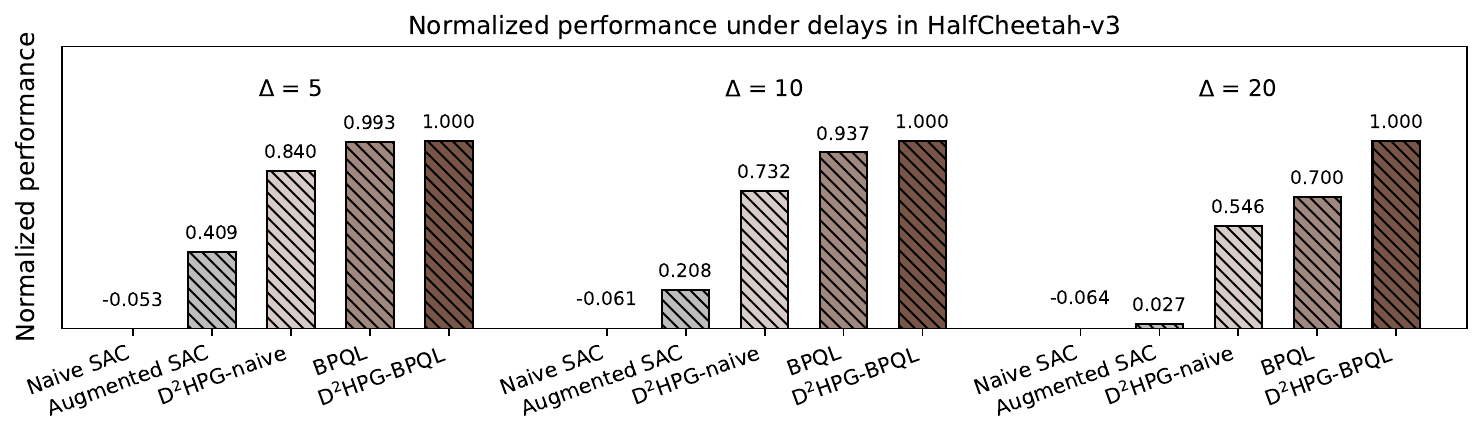}
        \caption{HalfCheetah-v3}
    \end{subfigure}    
    \begin{subfigure}[t]{1.0\textwidth}
        \centering
        \includegraphics[width=\textwidth]{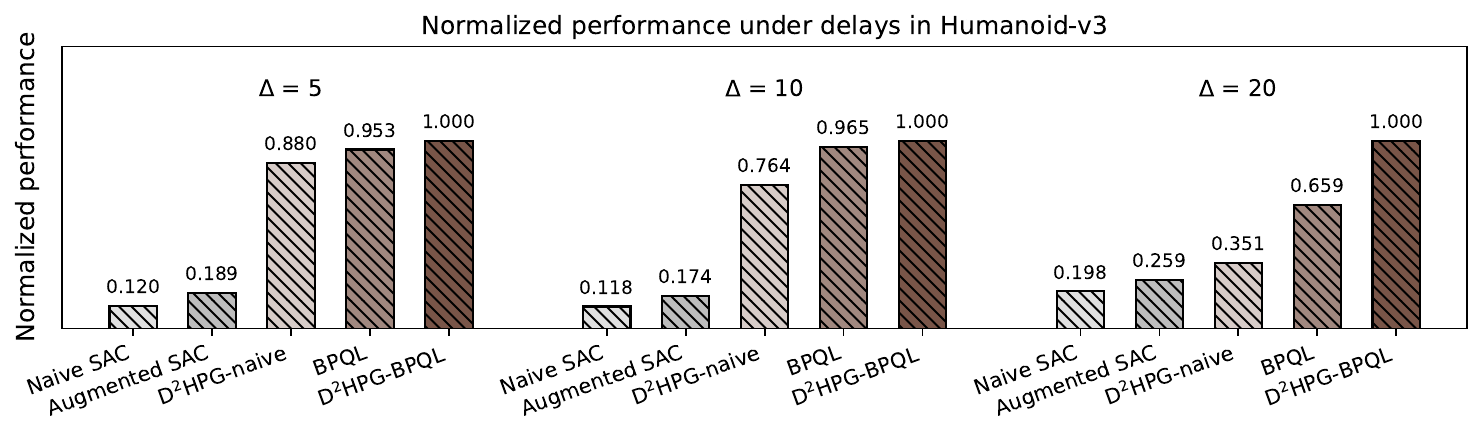}
        \caption{Humanoid-v3}
    \end{subfigure}    
    \caption{Normalized performance of D$^2$HPG-naive, D$^2$HPG-BPQL, and delayed RL baselines across three tasks in the MuJoCo benchmark. Each algorithm was evaluated for one million time steps with five random seeds, and performance was normalized with respect to D$^2$HPG-BPQL.}
    \label{fig:variants}
\end{figure*}

\newpage

\subsection{Sample Overhead} 


Discovering control-redundant structure directly in the augmented state space may require additional transition data and computation, potentially offsetting the sample-efficiency gains that the proposed abstraction mechanism is intended to provide. To overcome this practical limitation, we proposed practical D$^2$HPG that simplifies the principled D$^2$HPG by eliminating the explicit abstraction learning. 

To quantify the associated sample overhead, we compare the principled and practical variants of D$^2$HPG on MuJoCo tasks across ten random seeds under varying delay lengths. Specifically, we measure the number of additional time steps required for principled D$^2$HPG to match the performance achieved by practical D$^2$HPG after one million time steps. As shown in Fig.~\ref{fig:sample-overhead}, the sample overhead increases with the delay $\Delta$. These results suggest that learning the abstraction components over $\mathcal{X}$ can incur substantial sample overhead, highlighting the sample-efficiency benefit of practical D$^2$HPG.

\begin{figure}[h]
    \centering
    \includegraphics[width=0.435\columnwidth]{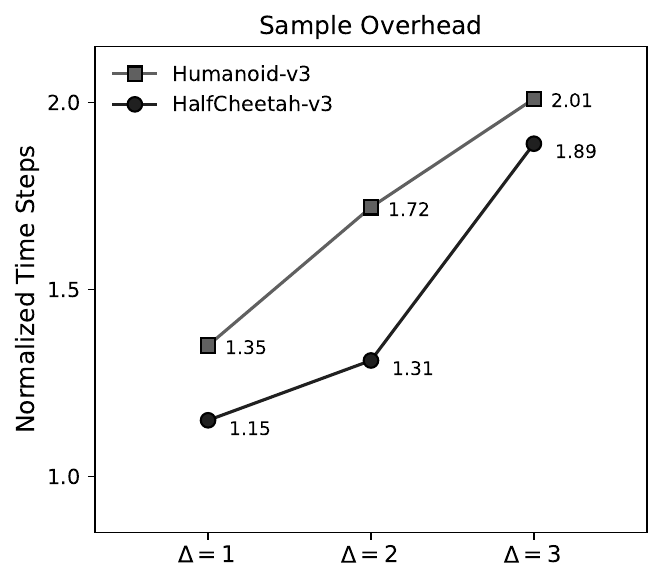}
    \caption{Normalized sample overhead of the principled D$^2$HPG for increasing delays $\Delta$, measured by the number of time steps required to match the performance achieved by the practical D$^2$HPG after one million time steps. Results are averaged over ten random seeds.}
    \label{fig:sample-overhead}
\end{figure}

\subsection{Computational Overhead} 

We evaluate the computational overheads of D$^2$HPG-naive and D$^2$HPG-BPQL by comparing their wall-clock runtimes with those of delayed RL baselines. Runtimes are measured over one million time steps in HalfCheetah-v3 using an NVIDIA RTX 3060 Ti GPU and an Intel i7-12700KF CPU, with the results reported in Fig.~\ref{runtimes-fig}. Compared with BPQL, D$^2$HPG-BPQL incurs additional computational overhead, reflecting a trade-off between improved sample efficiency and increased computational cost. Nevertheless, D$^2$HPG-BPQL remains substantially faster than VDPO.

\begin{figure*}[!h] 
    \centering
    \includegraphics[width=0.59\textwidth]{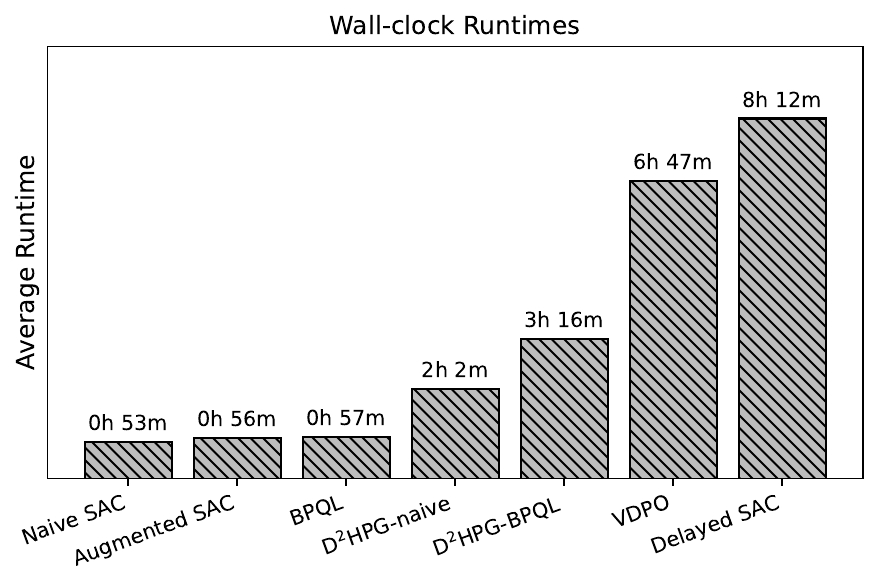}
    \caption{Wall-clock runtimes of D$^2$HPG and delayed RL baselines over one million time steps.} \label{runtimes-fig}
\end{figure*} 




\end{document}